\documentclass[english]{article}
\usepackage[LGR,T1]{fontenc}
\usepackage[latin9]{inputenc}
\usepackage{color}
\usepackage{url}
\usepackage{amsmath}
\usepackage{amsthm}
\usepackage{amssymb}
\usepackage{graphicx}

\makeatletter

\DeclareRobustCommand{\greektext}{%
  \fontencoding{LGR}\selectfont\def\encodingdefault{LGR}}
\DeclareRobustCommand{\textgreek}[1]{\leavevmode{\greektext #1}}
\ProvideTextCommand{\~}{LGR}[1]{\char126#1}

\theoremstyle{plain}
\newtheorem{thm}{\protect\theoremname}
\theoremstyle{plain}
\newtheorem{assumption}[thm]{\protect\assumptionname}
\theoremstyle{remark}
\newtheorem{rem}[thm]{\protect\remarkname}
\theoremstyle{plain}
\newtheorem{prop}[thm]{\protect\propositionname}

\usepackage{hyperref}       
\usepackage{url}            
\usepackage{booktabs}       
\usepackage{amsfonts}       
\usepackage{nicefrac}       
\usepackage{microtype}      
\usepackage{placeins}
\usepackage{bm}

\newcommand{\tabincell}[2]{\begin{tabular}{@{}#1@{}}#2\end{tabular}}

\PassOptionsToPackage{numbers, compress}{natbib}
\usepackage[final]{neurips_2020}

\usepackage{algorithm}
\usepackage{algorithmic}

\usepackage{caption}

\author{%
  Hui Chen$^*$ \\
  School of Mathematical Science\\
  Tongji University, Shanghai, P. R. China\\
  \texttt{hui.chen96@outlook.com} \\
  \And
  Fangqing Liu\thanks{Equal Contribution} \\
  School of Mathematical Science\\
  Tongji University, Shanghai, P. R. China\\
  \texttt{fangqingliu0@gmail.com}
  \And
  Yin Wang \\ 
  School of Electronics and Information Engineering\\
  Tongji University, Shanghai, P. R. China\\
  \texttt{yinw@tongji.edu.cn}
  \And
  Liyue Zhao \\
  Cloudwalk Inc.\\
  Shanghai, P. R. China \\
  \texttt{zhaoliyue@cloudwalk.cn} \\
  \And
  Hao Wu\thanks{Corresponding Author}\\
  School of Mathematical Science\\
  Tongji University, Shanghai, P. R. China \\
  \texttt{hwu@tongji.edu.cn} \\
}

\makeatother

\usepackage{babel}
\providecommand{\assumptionname}{Assumption}
\providecommand{\propositionname}{Proposition}
\providecommand{\remarkname}{Remark}
\providecommand{\theoremname}{Theorem}

\begin{document}
\title{A Variational Approach for Learning from Positive and Unlabeled Data}

\maketitle

\begin{abstract}
Learning binary classifiers only from positive and unlabeled (PU)
data is an important and challenging task in many real-world applications,
including web text classification, disease gene identification and
fraud detection, where negative samples are difficult to verify experimentally.
Most recent PU learning methods are developed based on the misclassification
risk of the supervised learning type, and they may suffer from inaccurate
estimates of class prior probabilities. In this paper, we introduce
a variational principle for PU learning that allows us to quantitatively
evaluate the modeling error of the Bayesian classifier directly from
given data. This leads to a loss function which can be efficiently
calculated without involving class prior estimation or any other intermediate
estimation problems, and the variational learning method can then
be employed to optimize the classifier under general conditions. We
illustrate the effectiveness of the proposed variational method on
a number of benchmark examples.
\end{abstract}

\section{Introduction\label{sec:Introduction}}

In many real-life applications, we are confronted with the task of
building a binary classification model from a number of positive data
and plenty of unlabeled data without extra information on the negative
data. For example, it is common in disease gene identification \cite{yang2012positive}
that only known disease genes and unknown genes are available, because
the reliable non-disease genes are difficult to obtain. Similar scenarios
occur in deceptive review detection \cite{ren2014positive}, web data
mining \cite{liu2007web}, inlier-based outlier detection \cite{smola2009relative},
etc. Such a task is certainly beyond the scope of the standard supervised
machine learning, and where \emph{positive-unlabeled} (PU) learning
comes in handy. A lot of heuristic approaches \cite{liu2002partially,peng2008svm,lu2010semi,chaudhari2012learning}
were proposed by identifying reliable negative data from the unlabeled
data, which heavily rely on the choice of the heuristic strategies
and the assumption of data separability (i.e., positive and negative
data are non-overlapping). The rank pruning (RP) \cite{northcutt2017learning}
provides a more general way by regarding the PU learning as a specific
positive-negative learning problem with noisy labels , but the data
separability is still necessary for the consistent noise estimation.

The risk estimator developed in \cite{du2014analysis,du2015convex}
promises an effective solution to PU learning. It calculates the risk
of a classifier $\Phi$ by
\begin{equation}
\mathrm{risk}(\Phi)=\pi_{P}\mathbb{E}_{\text{labeled data}}\left[\ell_{+}\left(\Phi(x)\right)-\ell_{-}\left(\Phi(x)\right)\right]+\mathbb{E}_{\text{unlabeled data}}\left[\ell_{-}\left(\Phi(x)\right)\right]\label{eq:unbiased-risk}
\end{equation}
and can achieve an unbiased estimation of the expected misclassification
risk (in the sense of supervised learning) via empirical averaging,
where $\ell_{+},\ell_{-}$ denotes the misclassification loss on positive
and negative data respectively, and $\pi_{P}=\mathbb{P}(y=+1)$ denotes
the \emph{class prior}, i.e., the proportion of positive data in the
unlabeled data. Then the classifier can be trained through minimization
of the estimated risk when $\pi_{P}$ is known. However, such a method
easily leads to severe overfitting. In order to address this difficulty,
a non-negative risk estimator is presented in \cite{kiryo2017positive},
which is biased but more robust to statistical noise. Another type
of misclassification risk based method, called PULD, was proposed
in \cite{zhang2019positive}, where PU learning is formulated as a
maximum margin classification problem for a given $\pi_{P}$, and
can be solved by efficient convex optimizers. But this method is applicable
only for linear classifiers in non-trainable feature spaces.

Recently, applications of generative adversarial networks (GAN) in
PU learning also have received growing attention \cite{hou2017generative,chiaroni2018learning},
where the generative models learn to generate fake positive and negative
samples (or only negative samples), and the classifier is trained
by using the fake samples. Experiments show that GAN can improve the
performance of PU learning when the size of positive labeled data
is extremely small, and the asymptotic correctness can be proved under
the condition that the exact value of $\pi_{P}$ is available \cite{hou2017generative}.

\paragraph{Problems of class prior estimation}

The class prior $\pi_{P}$ plays an important role in PU learning
as analyzed previously, but it cannot be automatically selected as
a trainable parameter. As an example, when trying to minimize the
risk defined in (\ref{eq:unbiased-risk}) w.r.t.~both $\pi_{P}$
and the classifier, we obtain a trivial solution with $\pi_{P}=1$
and all data being predicted as positive ones. Furthermore, it is
also difficult to adjust $\pi_{P}$ as a hyper-parameter by cross
validation unless some negative data are available in the validation
set. Hence, in many practical applications, class prior estimation
methods \cite{jain2016nonparametric,christoffel2016class,du2017class,bekker2018estimating}
are required, which usually involve kernel machines and are quite
computationally costly. Moreover, the experimental analysis in \cite{kiryo2017positive}
shows that the classification performance could be badly affected
by an inaccurate estimate.

\paragraph{Contributions}

In view of the above remark, it is natural to ask if \emph{an accurate
classifier can be obtained in PU learning without solving the hard
class prior estimation problem as an intermediate step}. Motivated
by this question, we introduce in this paper a variational principle
for PU learning, which allows us to evaluate the difference between
a given classifier and the ideal Bayesian classifier in a class prior-free
manner by using only distributions of labeled and unlabeled data.
As a consequence, one can efficiently and consistently approximate
Bayesian classifiers via variational optimization. Our theoretical
and experimental analysis demonstrates that, in contrast with the
existing methods, the variational principle based method can achieve
high classification accuracies in PU learning tasks without the estimation
of class prior or the assumption of data separability. A brief algorithmic
and theoretical comparison of VPU and selected previous schemes is
provided in Table \ref{tab:comparison}.

\begin{table}
\caption{\label{tab:comparison}A comparison of PU learning methods. Here uPU
and nnPU are proposed in \cite{du2014analysis,kiryo2017positive},
GenPU is presented in \cite{hou2017generative}, rank pruning \cite{northcutt2017learning}
is developed within the framework for classification with noisy labels,
the Rocchio-SVM method proposed in \cite{li2003learning} is a representative
method developed based on identification of reliable negative data,
and PULD \cite{zhang2019positive} is proposed based on the large
margin strategy. Rank pruning can be implemented with unknown class
prior, but it contains an estimator for class prior explicitly and
the estimator is consistent only in the case of data separability.}

\centering{}\begin{center}
\begin{tabular}{ccc} 
\toprule 
Method  & \tabincell{c}{Training without class prior\\ $\mathbb P(y=+1)$ or its estimate} &  \tabincell{c}{Consistency or optimality without\\ assumption of data separability}\\ 
\midrule
VPU & $\checkmark$  & $\checkmark$\\
\midrule
uPU/nnPU & $\times$ & $\checkmark$\\
\midrule
GenPU & $\times$ &  $\times$\\
\midrule
Rank pruning &  $\checkmark$ & $\times$\\
\midrule
\tabincell{c}{Rocchio-SVM} &  $\checkmark$ &  $\times$\\
\midrule
PULD & $\times$ & $\checkmark$\\
\bottomrule 
\end{tabular} 
\end{center}
\end{table}

\section{Problem setting and notations\label{sec:Problem-setting}}

Let us consider a binary classification problem where features $x\in\mathbb{R}^{d}$
and class labels $y\in\{-1,+1\}$ of instances are distributed according
to a joint distribution $\mathbb{P}(x,y)$. Suppose that we have a
positive dataset $\mathcal{P}=\{x_{1},\ldots,x_{M}\}$ and an unlabeled
dataset $\mathcal{U}=\{x_{M+1},\ldots,x_{M+N}\}$. The goal of PU
learning is to find a binary classifier based on $\mathcal{P}$ and
$\mathcal{U}$ so that class labels of unseen instances can be accurately
predicted. In this work, we aim to approximate the ideal Bayesian
classifier $\Phi^{*}(x)\triangleq\mathbb{P}(y=+1|x)$ with a parametric
model $\Phi$ based on the following assumptions:
\begin{assumption}
\label{assu:scar}Labeled and unlabeled data are independently drawn
as
\begin{equation}
\mathcal{P}=\{x_{i}\}_{i=1}^{M}\stackrel{\mathrm{i.i.d}}{\sim}f_{P},\quad\mathcal{U}=\{x_{i}\}_{i=M+1}^{M+N}\stackrel{\mathrm{i.i.d}}{\sim}f
\end{equation}
where $f_{P}\triangleq\mathbb{P}(x|y=+1)$ is the distribution of
the positive class and $f(x)\triangleq\mathbb{P}(x)$ denotes the
marginal distribution of the instance feature.
\end{assumption}

\begin{assumption}
\label{assu:max-Phi}There exists a set $\mathcal{A}\subset\mathbb{R}^{d}$
satisfying $\int_{\mathcal{A}}f_{P}(x)\mathrm{d}x>0$ and 
\begin{equation}
\Phi^{*}(x)=1,\quad\forall x\in\mathcal{A}.\label{eq:max-Phi}
\end{equation}
\end{assumption}

Here, Assumption \ref{assu:scar} is the traditional \emph{selected
completely at random} (SCAR) assumption in PU learning \cite{elkan2008learning,du2015convex}.
Assumption \ref{assu:max-Phi} implies that a set of $x$ are almost
surely positive, which is approximately satisfied in most practical
cases and actually a strong variant of the \textit{irreducibility}
assumption in literature of mixture proportion estimation of PU data
\cite{yao2020towards} (see Section \ref{subsec:Assumption-and-irreducibility}
in Suppl.~Material). In practice, $\mathcal{A}$ might be too small
and $\mathcal{P}$ is finite, so $\mathcal{A}\cap\mathcal{P}$ could
be empty. Thus we analyze the misclassification rate under a relaxation
of Assumption \ref{assu:max-Phi} (see Section \ref{sec:Theoretical-analysis}).

\section{Variational PU learning}

\subsection{Variational principle\label{subsec:Variational-principle}}

In this section we establish a novel variational principle for PU
learning without class prior estimation that will be used in rest
of this paper. According to the Bayes rule, for a given parametric
model $\Phi$ of the Bayesian classifier $\Phi^{*}$, the positive
data distribution $f_{P}$ can be approximated by
\begin{eqnarray}
f_{P}(x) & = & \frac{\mathbb{P}(y=+1|x)\mathbb{P}(x)}{\int\mathbb{P}(y=+1|x)\mathbb{P}(x)\mathrm{d}x}\nonumber \\
 & \approx & \frac{\Phi(x)f(x)}{\mathbb{E}_{f}[\Phi(x)]}\triangleq f_{\Phi}(x),\label{eq:f_phi}
\end{eqnarray}
and we can further prove that $f_{\Phi}=f_{P}$ if and only if $\Phi=\Phi^{*}$
under Assumptions \ref{assu:scar} and \ref{assu:max-Phi}.\footnote{All proofs can be found in Section \ref{sec:Analysis-of-VPU} in Suppl.~Material.}
Then, the approximation quality of $\Phi$ can be evaluated by some
divergence between $f_{P}$ and $f_{\Phi}$, e.g., the Kullback-Leibler
(KL) divergence $\mathrm{KL}(f_{P}||f_{\Phi})$. The above analysis
leads to our main theorem:
\begin{thm}
\label{thm:variational}For all $\Phi:\mathbb{R}^{d}\mapsto[0,1]$
with $\mathbb{E}_{f}\left[\Phi(x)\right]>0$,
\begin{equation}
\mathrm{KL}(f_{P}||f_{\Phi})=\mathcal{L}_{\mathrm{var}}(\Phi)-\mathcal{L}_{\mathrm{var}}(\Phi^{*}),\label{eq:KL}
\end{equation}
under Assumption \ref{assu:scar}, where\textsl{
\begin{equation}
\mathcal{L}_{\mathrm{var}}(\Phi)\triangleq\log\mathbb{E}_{f}\left[\Phi(x)\right]-\mathbb{E}_{f_{P}}\left[\log\Phi(x)\right].\label{eq:Lvar}
\end{equation}
}
\end{thm}

Since the KL divergence is always nonnegative, $\mathcal{L}_{\mathrm{var}}(\Phi)$
provides a variational upper bound of $\mathcal{L}_{\mathrm{var}}(\Phi^{*})$,
which can be easily computed by empirical averages over sets $\mathcal{P},\mathcal{U}$,
and the KL divergence of $f_{P}$ from $f_{\Phi}$ can be minimized
by equivalently minimizing $\mathcal{L}_{\mathrm{var}}(\Phi)$ (see
Fig.~\ref{fig:variational} for illustration). As a result, by selecting
a regularization functional $\mathcal{L}_{\mathrm{reg}}$ (see Section
\ref{subsec:Regularized-learning-method}), parameters of $\Phi$
can be optimized by solving

\begin{equation}
\min_{\Phi}\mathcal{L}(\Phi)=\mathcal{L}_{\mathrm{var}}(\Phi)+\lambda\mathcal{L}_{\mathrm{reg}}(\Phi)\label{eq:L}
\end{equation}
subject to constraints $\Phi(x)\in[0,1]$ and $\max_{x}\Phi(x)=1$.
In what follows, we refer to such a method variational PU (VPU) learning.

\begin{figure}[t]
\begin{centering}
\includegraphics[scale=0.4]{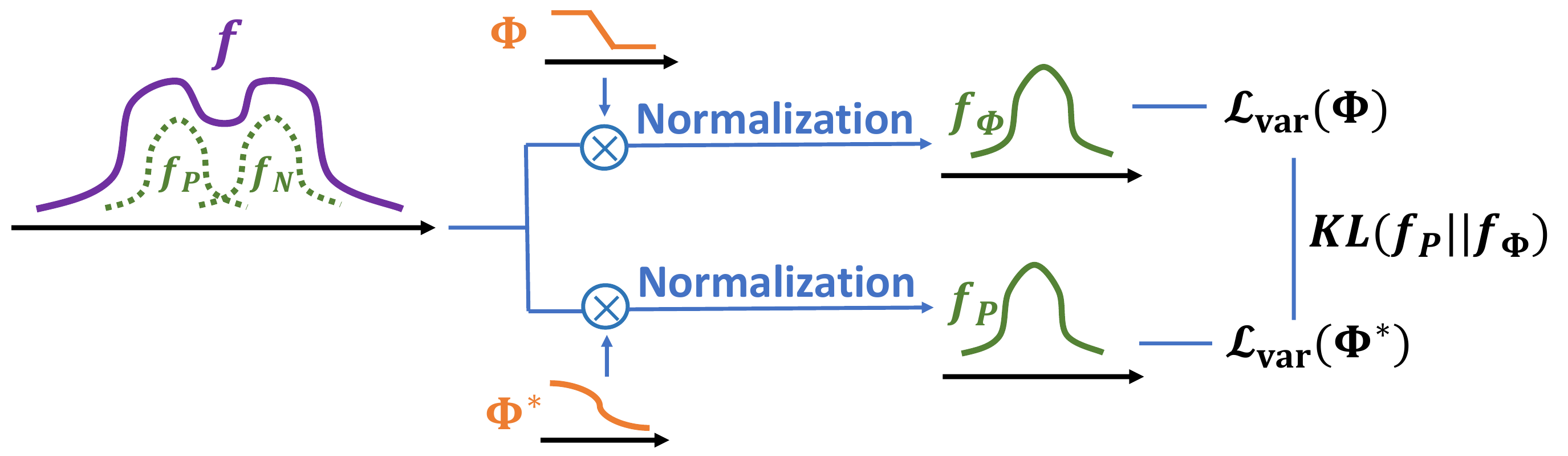}
\par\end{centering}
\caption{\label{fig:variational}Graphical interpretation of the variational
principle stated by Theorem \ref{thm:variational}, where $f_{P},f_{N},f$
denote distributions of positive, negative and unlabeled data. Each
classifier model $\Phi$ induces an approximation $f_{\Phi}$ of $f_{P}$
as in (\ref{eq:f_phi}), and $\mathrm{KL}(f_{P}||f_{\Phi})$ equals
to the difference between functionals $\mathcal{L}_{\mathrm{var}}(\Phi)$
and $\mathcal{L}_{\mathrm{var}}(\Phi^{*})$.}
\end{figure}

\begin{rem}
Theorem \ref{thm:variational} can also be interpreted as a corollary
to the Donsker-Varadhan representation theorem \cite{donsker1975asymptotic,belghazi2018mine}
by utilizing the variational representation of $\mathrm{KL}(f_{P}||f)$.
Based on the Donsker-Varadhan representation, objective functions
similar to $\mathcal{L}_{var}$ have been proposed to tackle various
problems, such as estimation of mutual information \cite{belghazi2018mine},
density ratio estimation \cite{tsuboi2009direct} and identification
of information-leaking features \cite{hsu2019obfuscation}.\end{rem}\begin{rem}

\label{rem:why_normalization}Although $\mathcal{L}_{\mathrm{var}}$
is scalar invariant with $\mathcal{L}_{\mathrm{var}}(c\cdot\Phi)=\mathcal{L}_{\mathrm{var}}(\Phi)$
for $c>0$ according to (\ref{eq:Lvar}), $\Phi^{*}$ can be uniquely
determined by the variational principle due to Assumption \ref{assu:max-Phi}
(see Section \ref{sec:Analysis-of-VPU} in Suppl.~Material).
\end{rem}

\subsection{Regularized learning method\label{subsec:Regularized-learning-method}}

The variational principle provides an asymptotically correct way to
model the classifier for PU learning in the limit of infinite sizes
of $\mathcal{P}$ and $\mathcal{U}$. However, in many application
scenarios, the size of labeled data is too small and the empirical
distribution cannot represent $f_{P}$. Therefore, simply minimizing
$\mathcal{L}_{\mathrm{var}}(\Phi)$ with a complex $\Phi$ may suffer
from overfitting and yield underestimation of $\Phi^{*}(x)$ for positive
but unlabeled data that are not close-neighbors of labeled data (see
analysis in Section \ref{sec:Analysis-of-VPU} of Suppl.~Material).

To overcome the above-mentioned issue of non-robustness, we incorporate
a \textit{MixUp} \cite{zhang2017mixup} based consistency regularization
term to the variational loss (\ref{eq:L}) as
\begin{equation}
\mathcal{L}_{\mathrm{reg}}(\Phi)=\mathbb{E}_{\tilde{\Phi},\tilde{x}}\left[\left(\log\tilde{\Phi}-\log\Phi(\tilde{x})\right)^{2}\right],\label{eq:reg}
\end{equation}
with
\begin{eqnarray}
\gamma & \stackrel{\mathrm{iid}}{\sim} & \mathrm{Beta}(\alpha,\alpha),\nonumber \\
\tilde{x} & = & \gamma\cdot x'+(1-\gamma)\cdot x'',\nonumber \\
\tilde{\Phi} & = & \gamma\cdot1+(1-\gamma)\cdot\Phi(x'').\label{eq:mixup}
\end{eqnarray}
Here $\tilde{x}$ is a sample generated by mixing randomly selected
$x'\in\mathcal{P}$ and $x''\in\mathcal{U}$, and $\tilde{\Phi}$
represents the \emph{guessed} probability $\mathbb{P}(y=+1|x=\tilde{x})$
constructed by the linear interpolation of the true label and that
predicted by $\Phi$. The consistency regularization is popular for
semi-supervised learning methods \cite{xie2019unsupervised,berthelot2019mixmatch},
and encourages smoothness of the model $\Phi$ especially in the area
between labeled and unlabeled data in VPU. Unlike in \cite{berthelot2019mixmatch},
here we perform MixUp between labeled and unlabeled samples, and quantify
the consistency between the predicted and interpolated $\Phi(\tilde{x})$
by the mean squared logarithmic error rather than the mean squared
error used in \cite{berthelot2019mixmatch}, because this scheme penalizes
more heavily the underestimation of $\Phi(\tilde{x})$ (see Section
\ref{sec:Analysis-of-VPU} in Suppl.~Material for detailed analysis).
The effectiveness of the proposed consistency regularization is validated
by our ablation study in Section \ref{subsec:ablation-study}. Finally,
it is noteworthy that some other regularization schemes without data
augmentation can also work well in the VPU framework (see, e.g., Section
\ref{subsec:Alternative-regularization-terms}).

A stochastic gradient based implementation of VPU with loss function
defined by (\ref{eq:L}) and (\ref{eq:reg}) is given in Algorithm
\ref{alg:VPU}, where regularization parameters $\lambda$ and $\alpha$
can be tuned by comparing the variational loss $\mathcal{L}_{\mathrm{var}}(\Phi)$
on the validation set.

\begin{algorithm}[tb]
\caption{Stochastic gradient based VPU}

\label{alg:VPU}

\begin{algorithmic}[1]

\STATE {\bfseries Input:} Positive and negative data sets $\ensuremath{\mathcal{P},\mathcal{U}}$,
a parametric model of $\Phi:\mathbb{R}^{d}\mapsto[0,1]$, hyperparameters
$\ensuremath{\lambda}$ and $\ensuremath{\alpha}$.

\REPEAT

\STATE Randomly sample mini-batches $\mathcal{B}^{\mathcal{P}}$
and $\ensuremath{\mathcal{\mathcal{B}^{\mathcal{U}}}}$ from $\ensuremath{\mathcal{P}}$
and $\ensuremath{\mathcal{U}}$ with batch size $B$.

\STATE Compute the variational loss by
\[
\hat{\mathcal{L}}_{var}=\log\frac{\sum_{x\in\ensuremath{\mathcal{\mathcal{B}^{\mathcal{U}}}}}\Phi(x)}{B}-\frac{\sum_{x\in\mathcal{B}^{\mathcal{P}}}\log\Phi(x)}{B}.
\]

\STATE Sample $\gamma\sim\text{Beta}\left(\alpha,\alpha\right)$,
and perform MixUp between labeled and unlabeled data by for $i=1,\cdots,B$
\begin{align*}
\tilde{x}_{i} & =\gamma x_{i}^{\mathcal{P}}+\left(1-\gamma\right)x_{i}^{\mathcal{U}}\\
\tilde{\Phi}_{i} & =\gamma+\left(1-\gamma\right)\Phi\left(x_{i}^{\mathcal{U}}\right),
\end{align*}
where $x_{i}^{\mathcal{P}}\in\mathcal{B}^{\mathcal{P}},x_{i}^{\mathcal{U}}\in\mathcal{\mathcal{B}^{\mathcal{U}}}$.

\STATE Compute the regularization term and the total loss by
\begin{align*}
\hat{\mathcal{L}}{}_{\mathrm{reg}} & =\frac{1}{B}\sum_{i=1}^{B}\left[\log\Phi\left(\tilde{x}_{i}\right)-\log\tilde{\Phi}_{i}\right]^{2},\\
\hat{\mathcal{L}} & =\hat{\mathcal{L}}+\lambda\hat{\ell}{}_{\mathrm{reg}}.
\end{align*}

\STATE Update parameters $W$ of $\Phi$ with step-sizes $\eta$
as
\[
W\leftarrow W-\eta\frac{\partial\hat{\mathcal{L}}}{\partial W}.
\]

\UNTIL The terminal condition is satisfied.

\STATE Perform the normalization
\[
\Phi(x)\leftarrow\min\left\{ \frac{\Phi(x)}{\max_{x\in\mathcal{P}\cup\mathcal{U}}\Phi(x)},1\right\} 
\]
according to  Remark \ref{rem:why_normalization}.

\end{algorithmic}

\end{algorithm}

\subsection{Comparison with related work}

From an algorithmic perspective, VPU is similar to the risk estimator
based PU learning methods, uPU and nnPU \cite{du2014analysis,kiryo2017positive}.
All the three methods optimize parameters of the classifier with respect
to some empirical estimates of a loss under the SCAR assumption, and
the major difference comes from the fact that by introducing Assumption
\ref{assu:max-Phi}, VPU can be implemented without the class prior
$\pi_{P}$. As analyzed in Section \ref{sec:Problem-setting} and
Section \ref{sec:Analysis-of-VPU} in Suppl.~Material, Assumption
2 comprises most practical cases where some instances are positive
with probability one, and most class prior estimation methods require
similar irreducibility assumptions for the identifiability of $\text{\ensuremath{\pi_{P}}}$.
It can also be proved that a slight relaxation of this assumption
will not significantly affect the asymptotic correctness of VPU (see
Theorem \ref{thm:bias}).

Furthermore, all hyperparameters in VPU, including regularization
parameters, model class and iteration number, can be determined by
$\mathcal{L}_{\mathrm{var}}$ based cross validation. But for uPU
and nnPU, because the estimated risks heavily rely on the class prior
$\pi_{P}$, choosing $\pi_{P}$ by the direct estimated risk based
cross validation will yield an uninformative result with $\pi_{P}=1$.
(See analysis in Section \ref{sec:Analysis-of-VPU} of Suppl.~Material.)

\section{Theoretical analysis \label{sec:Theoretical-analysis}}

The asymptotic correctness of VPU is a direct consequence of the variational
principle introduced in Section \ref{subsec:Variational-principle}
as shown in the following theorem.
\begin{thm}
\label{thm:correctness}Provided that the following conditions hold:
(i) Assumptions \ref{assu:scar} and \ref{assu:max-Phi} hold, (ii)
and the classifier is modeled as $\Phi(x)=\Phi(x,\theta)$ with parameters
$\theta$ and there exists $\theta^{*}$ so that $\Phi^{*}(x)=\Phi(x,\theta^{*})$.
Then the optimal $\Phi$ obtained by VPU satisfies $\Phi\stackrel{p}{\to}\Phi^{*}$
as $M,N\to\infty$ and $\lambda\to0$.
\end{thm}

We now analyze the effects of relaxation of Assumptions \ref{assu:scar}
and \ref{assu:max-Phi} on VPU.
\begin{assumption}
\label{assu:separability}$\mathcal{P}\stackrel{\mathrm{iid}}{\sim}f_{P}^{\prime}$,
$\mathcal{U}\stackrel{\mathrm{iid}}{\sim}f$, where $f_{P}^{\prime}$
differs from the positive data distribution $f_{P}$ in $\mathcal{U}$,
and $f_{P}^{\prime},f_{P}$ satisfy (i) there are positive constants
$c_{1},c_{2}$ close to $1$ so that $c_{1}f_{P}(x)\le f_{P}^{\prime}\le c_{2}f_{P}(x)$
and (ii) there is a set $\mathcal{A}\subset\mathbb{R}^{d}$ with $\int_{\mathcal{A}}f_{P}(x)\mathrm{d}x>0$
so that $\min_{x\in\mathcal{A}}\Phi^{*}(x)\ge1-\epsilon$ with $\epsilon\in[0,1)$
being a small number.
\end{assumption}

\begin{thm}
\label{thm:bias}If data distributions satisfy Assumption \ref{assu:separability},
the optimal $\Phi$ obtained by VPU with $\lambda=0$ and $M,N\to\infty$
satisfies
\[
\left|\mathcal{R}(\Phi)-\mathcal{R}(\Phi^{*})\right|\le\max\left\{ \frac{c_{2}}{c_{1}}-1,1-\frac{c_{1}(1-\epsilon)}{c_{2}}\right\} ,
\]
where $\mathcal{R}(\Phi)$ denotes the misclassification rate of the
predicted label $y=\mathrm{sign}(\Phi(x)-0.5)$.
\end{thm}

Selection bias is a practically important but theoretically challenging
classification problem for VPU, which implies that the labeled data
distribution $f_{P}^{\prime}$ may differ from the positive data distribution
$f_{P}$ \cite{bekker2018bias,kato2018learning}. Although the variational
principle in this case requires further investigations, Theorem \ref{thm:bias}
ensures that the VPU learning can still obtain a classification accuracy
comparable to the ideal $\Phi^{*}$, i.e., $\mathcal{R}(\Phi)\approx\mathcal{R}(\Phi^{*})$,
if the selection bias is limited with $c_{1},c_{2}$ close to $1$
and Assumption \ref{assu:separability} is only slightly violated
with $\epsilon\ll1$.Our numerical experiments also indicate that
the proposed VPU is quite robust to the bias of labeled data (see
Section \ref{subsec:selection-bias-experiments}).

\section{Experiments}

In this section, we test the effectiveness of VPU on both synthetic
and real-world datasets. We provide an extensive ablation study to
analyze the regularization defined by (\ref{eq:reg}). Considering
selection bias is common in practice, we test the effectiveness of
VPU and existing methods in this scenario. At last, we further demonstrate
the robustness of VPU by experiments with different size of the labeled
set.

\subsection{Implementation details}

The class label is predicted as $y=\mathrm{sign}(\Phi(x)-0.5)$ in
VPU when calculating classification accuracies. In all experiments,
$\alpha$ is chosen as $0.3$ and\textcolor{red}{{} }$\lambda\in\left\{ 1e-4,3e-4,1e-3,\cdots,1,3\right\} $
is determined by holdout validation unless otherwise specified. We
use Adam as the optimizer for VPU with hyperparameters $(\text{\ensuremath{\beta_{1}}},\beta_{2})=(0.5,0.99)$.

The performance of VPU is compared to that of some recently developed
PU learning methods, including the unbiased risk estimator based uPU
and nnPU \cite{du2014analysis,kiryo2017positive}, the generative
model based GenPU \cite{hou2017generative}, and the rank pruning
(RP) proposed in \cite{northcutt2017learning}.\footnote{The software codes are downloaded from \url{https://github.com/kiryor/nnPUlearning},
\url{https://qibinzhao.github.io/index.html} and \url{https://github.com/cgnorthcutt/rankpruning}.} Notice that uPU and nnPU require the prior knowledge of the class
proportion. Thus, for fair comparison, $\pi_{P}$ is estimated by
the KM2 method proposed in \cite{ramaswamy2016mixture} when implementing
uPU and nnPU, where KM2 is one of the state-of-the-art class prior
estimation algorithms. For GenPU, the hyperparameters of the algorithm
are determined by greedy grid search as described in Section \ref{sec:Experiments-details}
in Suppl.~Material.

In all the methods, the classifiers (including discriminators of GenPU)
are modeled by $7$-layer MLP for UCI datasets, LeNet-5 \cite{lecun1998gradient}
for FashionMNIST and $7$-layer CNN for CIFAR-10 and STL-10. By default,
the accuracies are evaluated on test sets and the mean and standard
deviation values are computed from $10$ independent runs. All the
other detailed settings of datasets and algorithms are provided in
Section \ref{sec:Experiments-details} of Suppl.~Material, and the
software code for VPU is also available\footnote{\url{https://github.com/HC-Feynman/vpu}}.

\subsection{Benchmark data \label{subsec:UCI-experiments}}

We conduct experiments on three benchmark datasets taken from the
UCI Machine Learning Repository \cite{Dua:2019,de2018reliable}, and
the classification results are reported in Table \ref{tab:UCI}. It
can be seen that VPU outperforms the other methods with high accuracies
and low variances on almost all the datasets. nnPU and uPU suffer
from the estimation error of $\pi_{P}$. In fact, if $\pi_{P}$ is
exactly given, nnPU can achieve better performance, though still a
little worse than VPU. (See Section \ref{sec:Experiments-details}
in Suppl.~Material.) In addition, RP interprets unlabeled data as
noisy negative data and can get an accurate classifier when the proportion
of positive data is small in unlabeled data. But in the opposite case
where the proportion is too large, RP performs even worse than random
guess. ($\pi_{P}$ = 0.896 and 0.635 in Page Blocks with\textcolor{black}{{}
'text' vs 'horizontal line, vertical line, picture, graphic} and Grid
Stability with \textquoteright unstable\textquoteright{} vs \textquoteright stable\textquoteright .)

\begin{table}[tbh]
\caption{\label{tab:UCI}Classification accuracies (\%) of compared methods
on UCI datasets. Definitions of labels ('Positive' vs 'Negative')
are as follows: Page Blocks$^1$: \textcolor{black}{'horizontal line
, vertical line, picture, graphic' vs 'text'. Page Blocks$^2$: 'text'
vs 'horizontal line , vertical line, picture, graphic'.} Grid Stability$^1$:
'stable' vs 'unstable'. Grid Stability$^2$: 'unstable' vs 'stable'.
Avila$^1$: 'A' vs the rest. Avila$^2$: 'A, F' vs the rest. Labeled
positive data are randomly selected from the training data with $M=100,1000,2000$
and $N=3284,6000,10430$.}

\begin{center}
\footnotesize
\begin{tabular}{cccccccccc}  
\toprule  
Dataset  
& Page Blocks$^1$ & Page Blocks$^2$ & Grid Stability$^1$&  Grid Stability$^2$ & Avila$^1$& Avila$^2$ \\
\midrule 
\footnotesize
VPU &
\scriptsize \bm{$93.6\pm0.4$} & \scriptsize \bm{$93.5\pm0.7}$ & \scriptsize \bm{$92.6\pm0.3}$ &  \scriptsize \bm{$89.5\pm0.5$} & \scriptsize \bm{$82.0\pm0.9$} & \scriptsize \bm{$87.2\pm0.5$} \\
\midrule 
\footnotesize
nnPU &
\scriptsize $93.4\pm1.1$ & \scriptsize$90.2\pm2.6$ &\scriptsize$80.8\pm2.5$ &\scriptsize$84.1\pm1.8$ &\scriptsize$73.3\pm2.0$ &\scriptsize $83.1\pm2.1$ \\
\midrule
\footnotesize
uPU &
\scriptsize $92.8\pm1.3$ &\scriptsize$86.8\pm4.7$ &\scriptsize$\bm{92.6\pm0.7}$ &\scriptsize$86.8\pm0.5$ &\scriptsize$75.0\pm0.4$ &\scriptsize$82.7\pm1.7$\\
\midrule
\footnotesize
GenPU &
\scriptsize $93.2\pm0.3$ &\scriptsize$90.2\pm0.1$ &\scriptsize$69.3\pm0.6$  &\scriptsize$75.6\pm1.8$ &\scriptsize$63.4\pm1.1$  &\scriptsize$67.1\pm0.8$\\
\midrule
\footnotesize
RP 
&\scriptsize $91.2\pm1.4$ &\scriptsize$9.96\pm0.7$ &\scriptsize$84.7\pm1.3$ &\scriptsize$36.7\pm0.6$ &\scriptsize$75.8\pm0.4$ &\scriptsize$77.2\pm0.2$\\
\bottomrule 
\end{tabular}
\end{center}
\end{table}

\subsection{Image datasets\label{subsec:Image-datasets}}

Here we compare all the methods on three image datasets: FashionMNIST,
CIFAR-10, and STL-10. Notice that in the rest of the paper, we denote
the $10$ classes of each image datasets with integers ranging from
$0$ to $9$, following the default settings in torchvision 0.5.0
(see Section \ref{sec:Experiments-details} in Suppl. Material).\footnote{Datasets are downloaded from \url{https://github.com/zalandoresearch/fashion-mnist},
\url{https://www.cs.toronto.edu/~kriz/cifar.html} and \url{http://cs.stanford.edu/~acoates/stl10}.}The classification accuracies are collected in Table \ref{tab:image data},
in which the superiority of VPU is also marked  (see Section \ref{subsec:Other-metric-for}
for other comparison metric). Here uPU performs much worse than nnPU
due to the overfitting problem \cite{kiryo2017positive}. Moreover,
the performance of GenPU is also not satisfying because of the mode
collapse of generators (see Section \ref{sec:Experiments-details}
in Suppl.~Material).

\begin{table}[tbh]
\caption{\label{tab:image data}Classification accuracies (\%) of compared
methods on FashionMNIST (abbreviated as ``F-MNIST''), CIFAR-10 and
STL-10 datasets. Definitions of labels ('Positive' vs 'Negative')
are as follows: FashionMNIST$^1$: '1,4,7' vs '0,2,3,5,6,8,9'. FashionMNIST$^2$:
'0,2,3,5,6,8,9' vs '1,4,7'. CIFAR-10$^1$: '0,1,8,9' vs '2,3,4,5,6,7'.
CIFAR-10$^2$: '2,3,4,5,6,7' vs '0,1,8,9'. STL-10$^1$: '0,2,3,8,9'
vs '1,4,5,6,7'. STL-10$^2$: '1,4,5,6,7' vs '0,2,3,8,9'. For FashionMNIST
and CIFAR-10, labeled positive data are randomly selected from the
training data with $M=3000$. For STL-10, $\mathcal{P}$ are defined
as all positive labeled data in the training set with $M=2500$.}

\begin{center}
\begin{tabular}{cccccccc} 
\toprule
Dataset  & F-MNIST$^1$ & F-MNIST$^2$ & CIFAR-10$^1$ & CIFAR-10$^2$ & STL-10$^1$ & STL-10$^2$ \\
\midrule 
\footnotesize
VPU & \scriptsize \bm{$92.7\pm0.3$} & \scriptsize \bm{$90.8\pm0.6$} & \scriptsize\bm{$89.5\pm0.1$} & \scriptsize\bm{$88.8\pm0.8$} & \scriptsize\bm{$79.7\pm1.5$} & \scriptsize \bm{$83.7\pm0.1$} \\
\midrule 
\footnotesize
nnPU & \scriptsize$90.8\pm0.6$ & \scriptsize$ 90.5\pm0.4$ & \scriptsize$85.6\pm2.3$ & \scriptsize$85.5\pm2.0$ & \scriptsize $78.3\pm1.2$ & \scriptsize$82.2\pm0.5$ \\
\midrule 
\footnotesize
uPU & \scriptsize$89.9\pm1.0$ & \scriptsize$78.6\pm1.3$ &\scriptsize$80.6\pm2.1$ & \scriptsize$72.9\pm3.2$ & \scriptsize$70.3\pm2.0$ & \scriptsize$74.0\pm3.0$\\
\midrule 
\footnotesize
Genpu & \scriptsize$47.8\pm1.0$ & \scriptsize$78.8\pm0.3$ & \scriptsize$67.6\pm0.9$ & \scriptsize$72.1\pm1.1$ & \scriptsize$65.1\pm1.0$ & \scriptsize$68.1\pm1.3$\\
\midrule 
\footnotesize
RP & \scriptsize$92.2\pm0.4$ & \scriptsize$75.9\pm0.6$ & \scriptsize$86.7\pm2.9$ & \scriptsize$77.8\pm2.5$ & \scriptsize$67.8\pm4.6$ & \scriptsize$68.5\pm5.7$ \\
\bottomrule 
\end{tabular}
\end{center}
\end{table}

\subsection{Ablation study\label{subsec:ablation-study}}

To justify our choice for the regularization term (\ref{eq:reg}),
we conduct an ablation study on FashionMNIST with '1, 4, 7' as positive
labels and $1000$ labeled samples. We compare (a) consistency regularization
(\ref{eq:reg}) adopted in this paper with $x'\in\mathcal{P}$ and
$x''\in\mathcal{U}$ as in (\ref{eq:mixup}), (b) $\mathcal{L}_{\mathrm{reg}}(\Phi)\equiv0$,
(c) regularization with MixUp on $\mathcal{P}$ data only, (d) regularization
with MixUp on $\mathcal{P}\cup\mathcal{U}$, where $x',x''$ are both
randomly selected from $\mathcal{P}\cup\mathcal{U}$, (e) consistency
loss defined by the mean squared error $\mathbb{E}_{\tilde{\Phi},\tilde{x}}[\left(\tilde{\Phi}-\Phi(\tilde{x})\right){}^{2}]$.
Results in Table \ref{tab:ablation-study} show the superiority of
(\ref{eq:reg}).

\begin{table}

\caption{\label{tab:ablation-study}Ablation study results on FashionMNIIST
with '1, 4, 7' as positive labels and 1000 labeled samples. (\ref{eq:reg})
is the regularization term we adopt, i.e., mean squared logarithmic
error with MixUp between $\mathcal{P}$ and $\mathcal{U}$.}

\begin{center}
\begin{tabular}{cccccc}
\toprule
\tabincell{c}{Ablation on\\regularization} & (\ref{eq:reg}) &\tabincell{c}{no\\regularization} &\tabincell{c}{(\ref{eq:reg}) with MixUp \\on $\mathcal{P}$ only} &  \tabincell{c}{(\ref{eq:reg}) with MixUp\\ on $\mathcal{P}\cup\mathcal{U}$ }& \tabincell{c}{(\ref{eq:reg})\\with MSE}\\
\midrule
Test accuracies& \scriptsize$91.3\pm0.4$  & \scriptsize$87.2\pm2.9$  & \scriptsize$90.3\pm0.3$  & \scriptsize$90.1\pm0.9$ &\scriptsize $90.0\pm0.2$ 
\\
\bottomrule
\end{tabular}
\end{center}
\end{table}

\subsection{Selection bias\label{subsec:selection-bias-experiments}}

In many practical situations, the assumption that the empirical distribution
of $f_{P}$ is consistent with the ground truth may not be satisfied.
Hence, in this section we compare the PU methods in Section \ref{subsec:UCI-experiments}
on FashionMNIST with '1, 4, 7' as positive labels under selection
bias of $\mathcal{P}$. In this experiment, the total number of labeled
data is fixed to $3000$, but selection among different positive labels
is biased. For positive labels '1, 4, 7', we denote corresponding
numbers of labeled data as $n_{1},n_{4},n_{7}$, which satisfy $n_{1}+n_{4}+n_{7}=3000$
and $n_{4}=n_{7}\le n_{1}$. (Note the three classes have the same
size in the whole data set.) Performance of the methods is compared
in Fig. \ref{fig:bias}, which shows that VPU has a superior robustness
to sample selection bias of $\mathcal{P}$ over other methods.\textcolor{red}{{}
}\textcolor{black}{Poor performance of nnPU is, to a large extent,
attributed to the difficulties of class prior estimation under selection
bias, and nnPU is robust if the accurate class prior is known (See
}Section \ref{sec:Experiments-details}\textcolor{black}{{} in Suppl.}~\textcolor{black}{Material).}

\subsection{Different size of the labeled set}

Considering that a big labeled positive set is usually inaccessible
in applications, we investigate performance of the PU methods with
small labeled set on FashionMNIST with '1, 4, 7' as positive labels.
The labeled set size ranges from $500$ to $3000$, and Figure \ref{fig:diff_size}
shows the robustness of VPU.

\begin{figure}[t]
  \begin{minipage}[l]{\textwidth}
    \begin{minipage}[t]{0.48\textwidth}
      \centering
      \centerline{\includegraphics[width=\columnwidth]{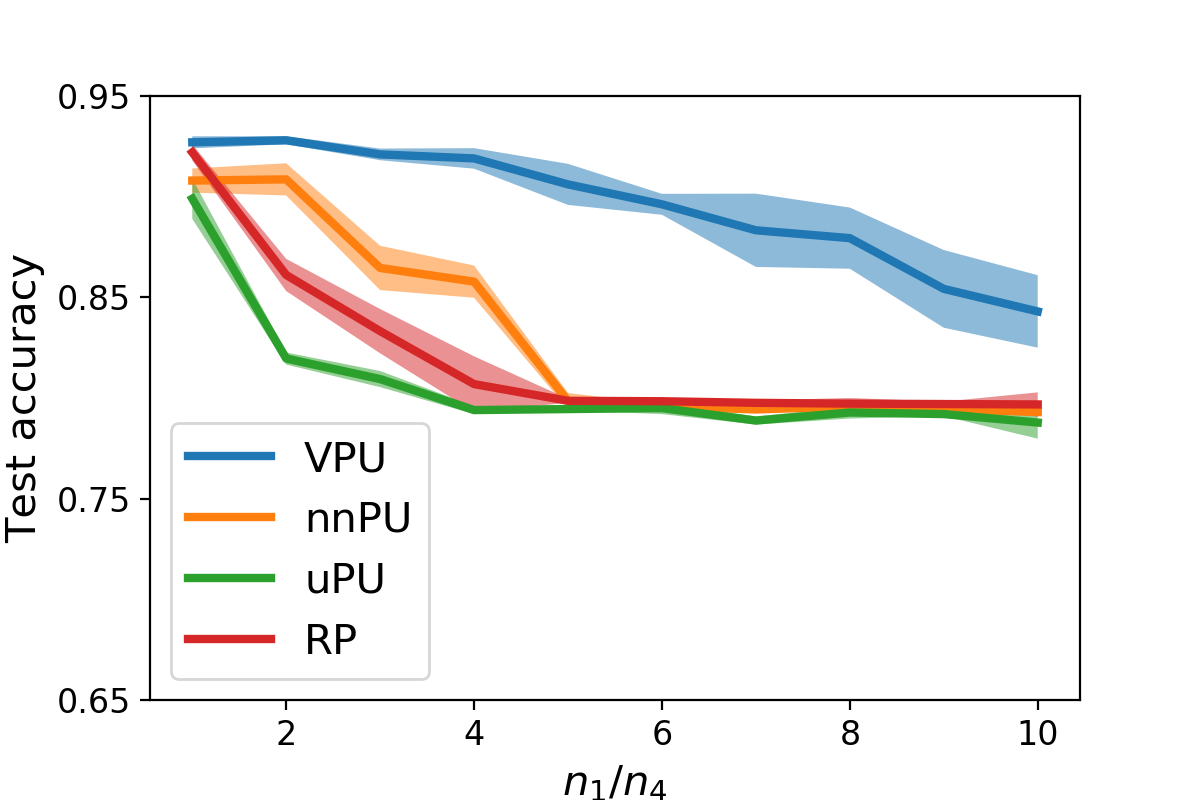}}
      \captionof{figure}{Test accuracies of PU methods on FashionMNIST with '1, 4, 7' as positive labels.  $n_1,n_4,n_7$   denote corresponding numbers of labeled samples for each label, with  $n_1+n_4+n_7=3000$ and $n_4=n_7$.}
      \label{fig:bias}
    \end{minipage}\hfill
    \begin{minipage}[t]{0.48\textwidth}
      \centering
      \centerline{\includegraphics[width=\columnwidth]{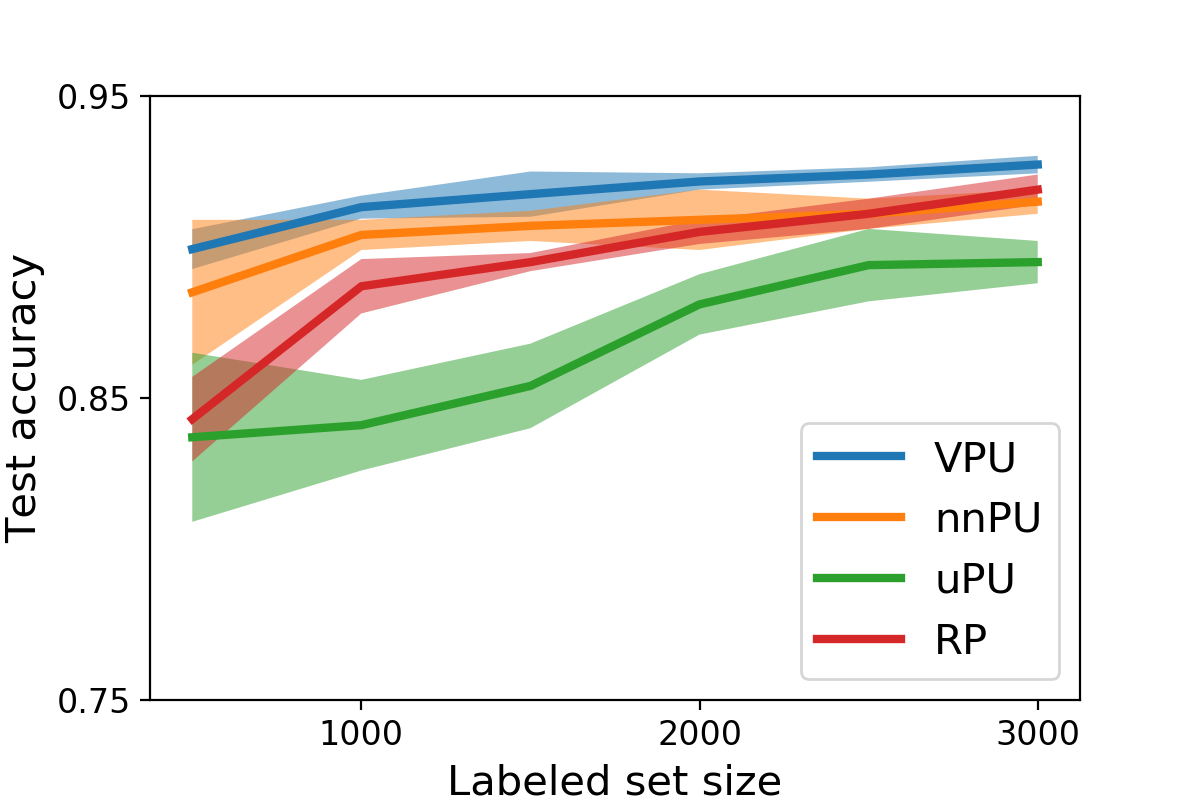}}
      \captionof{figure}{Test accuracies of PU methods on FashionMNIST with '1, 4, 7' as positive labels, with different size of the labeled set.}
    \label{fig:diff_size}
\end{minipage}
  \end{minipage}

\end{figure}

\section{Conclusion}

In this work, we proposed a novel variational principle for PU learning,
and developed an efficient learning method called variational PU (VPU).
In addition, a MixUp based regularization was utilized to improve
the stability of the method. We also showed that the method can consistently
estimate the optimal Bayesian classifier under a general condition
without any assumption on class prior or data separability. The superior
performance and robustness of VPU was confirmed in the experiments.

It is worthy to note that variational principle could be extended
to a more general framework by using different statistical distances,
and some other possible variational principles are discussed in Section
\ref{sec:Extension} of Suppl.~Material. Many advanced techniques
developed for measuring difference between distributions for GAN can
be expected to improve the performance of VPU.

\FloatBarrier

\section*{Broader Impact}

VPU is a general framework for PU learning, and it overcomes some
limitations of previous methods, including requirement of class prior
known beforehand and data separability, so is more applicable to real-world
applications. Thus discussion of the potential impacts of VPU actually
leads to the discussion of potential impacts of applications of PU
learning itself. With VPU, less labels are needed, which saves cost
and improves efficiency. Moreover, VPU is able to mine the negative
pattern that is missing in the PU datasets. This will be helpful if
finding out the negative pattern is beneficial, such as discovering
drugs for diseases and identifying deceptive reviews for recommendation
systems. However, malicious tasks can also be conducted with VPU,
such as discovery of harmful chemical substance. Another unethical
scenario is that sometimes the negative pattern could be hidden on
purpose for the sake of privacy or other ethical considerations, but
with VPU, people might be able to find out about the hidden information.

\section*{Acknowledgments and Disclosure of Funding}

The authors thank the anonymous NeurIPS reviewers for their valuable
feedback. Hao Wu is supported by the Fundamental Research Funds for
the Central Universities, China (No. 22120200276). Yin Wang is supported
by National Natural Science Foundation of China (No. 61950410614)
and Cross-disciplinary Program for the Central Universities, China
(No. 08002150042).

\bibliographystyle{ieeetr}
\bibliography{ref}

\clearpage

\appendix

\part*{Supplementary Material}

\section{Analysis of VPU\label{sec:Analysis-of-VPU}}

\subsection{Class prior estimation in uPU and nnPU}

The risk estimator in uPU \cite{du2014analysis,du2015convex} is defined
by (\ref{eq:unbiased-risk}), and nnPU \cite{kiryo2017positive} provides
a nonnegative estimator 
\begin{equation}
\mathrm{risk}(\Phi)=\pi_{P}\mathbb{E}_{\text{labeled data}}\left[\ell_{+}\left(\Phi(x)\right)\right]+\max\left\{ 0,\mathbb{E}_{\text{unlabeled data}}\left[\ell_{-}\left(\Phi(x)\right)\right]-\pi_{P}\mathbb{E}_{\text{labeled data}}\left[\ell_{-}\left(\Phi(x)\right)\right]\right\} \label{eq:nn-risk}
\end{equation}
in order to avoid overfitting, where the classifier $\Phi$ is not
necessarily an approximate Bayesian classifier and its range can be
$\mathbb{R}$. Both (\ref{eq:unbiased-risk}) and (\ref{eq:nn-risk})
consistently estimate the misclassification risk
\[
\pi_{P}\mathbb{E}_{f_{P}}\left[\ell_{+}\left(\Phi(x)\right)\right]+(1-\pi_{P})\mathbb{E}_{f_{N}}\left[\ell_{-}\left(\Phi(x)\right)\right]
\]
under Assumption \ref{assu:scar}, where $f_{N}$ denotes negative
distribution $\mathbb{P}(x|y=-1)$. In usual cases, loss functions
$\ell_{+}$ and $\ell_{-}$ satisfy \cite{kiryo2017positive}

\begin{enumerate}

\item $\ell_{+}\left(\Phi(x)\right)\ge0$ and $\ell_{-}\left(\Phi(x)\right)\ge0$
for all $x$.

\item $\ell_{+}\left(\Phi(x)\right)\to0$ as $\Phi(x)\to C$ for
some constant $C$, where $C$ can be $\infty$. This implies the
loss is zero if $\Phi$ classify a positive sample $x$ as positive
with a high confidence.

\end{enumerate}

If we minimize the estimated risk by regarding $\pi_{P}$ as a variable,
a trivial minimum of $0$ can be achieved with $\pi_{P}=1$ and $\Phi(x)\equiv C$
in the limit of infinite data size, i.e., all data are predicted as
positive. achieves a trivial minimum of $0$ with $\pi_{P}=1$ and
$\Phi\left(x\right)\equiv C$, i.e., unlabeled data are predicted
as positive. This result is obviously uninformative. Moreover, it
is also infeasible to select $\pi_{P}$ as a hyperparameter by the
estimated risk based cross validation, since the minimal estimated
risk on validation set can also be obtained with $\pi_{P}=1$ and
$\Phi\left(x\right)\equiv C$. Therefore, unless some negative samples
are available as validation data, the class prior estimation is an
unavoidable intermediate step when performing uPU or nnPU.

\subsection{\textit{\emph{Assumption \ref{assu:max-Phi} and irreducibility}}
assumption \label{subsec:Assumption-and-irreducibility}}

According to Assumption \ref{assu:scar}, the unlabeled data distribution
can be decomposed as
\[
f=\pi_{P}\cdot f_{P}+(1-\pi_{P})\cdot f_{N},
\]
where $f_{N}=\mathbb{P}(x|y=-1)$, and it can be rewritten as
\[
f=\pi_{P}^{\prime}\cdot f_{P}+(1-\pi_{P}^{\prime})f_{N}^{\prime}
\]
with
\begin{eqnarray*}
\pi_{P}^{\prime} & = & c\pi_{P},\\
f_{N}^{\prime} & = & \frac{\left(1-c\right)\pi_{P}\cdot f_{P}+(1-\pi_{P})\cdot f_{N}}{1-c\pi_{P}},
\end{eqnarray*}
for all $c\in(0,1)$. This implies that $f_{N}$ and $\pi_{P}$ cannot
be uniquely determined from $f,f_{P}$ if $f_{N}$ is a mixture distribution
which contains $f_{P}$. In order to deal with this problem, most
class prior estimation methods \cite{jain2016nonparametric,christoffel2016class,du2017class,bekker2018estimating}
assume that $f_{N}$ is irreducible with respect to $f_{P}$, i.e.,
if $f_{N}$ is not a mixture containing $f_{P}$ \cite{SI-ref-33}.
One stronger variant of the irreducibility assumption is \cite{SI-ref-34,SI-ref-35}
\begin{equation}
\min_{\mathcal{A}\subset\mathbb{R}^{d},\int_{\mathcal{A}}f_{P}(x)\mathrm{d}x>0}\frac{f_{N}(x)}{f_{P}(x)}=0.\label{eq:anchor}
\end{equation}

We now show that \textit{\emph{Assumption \ref{assu:max-Phi} is equivalent
to (\ref{eq:anchor}).}}
\begin{prop}
Assumption \ref{assu:max-Phi} is satisfied if and only if (\ref{eq:anchor})
holds.
\end{prop}

\begin{proof}
If (\ref{eq:anchor}) holds and $\mathcal{A}$ is an optimal solution,
\begin{eqnarray*}
\Phi^{*}(x) & = & \frac{\pi_{P}f_{P}(x)}{f(x)}\\
 & = & \frac{\pi_{P}f_{P}(x)}{\pi_{P}f_{P}(x)+(1-\pi_{P})f_{N}(x)}\\
 & = & \frac{1}{1+\frac{1-\pi_{P}}{\pi_{P}}\frac{f_{N}(x)}{f_{P}(x)}}\\
 & = & 1
\end{eqnarray*}
for all $x\in\mathcal{A}$, and therefore Assumption \ref{assu:max-Phi}
is satisfied by $\mathcal{A}$. If Assumption \ref{assu:max-Phi}
holds with set $\mathcal{A}$,
\begin{eqnarray*}
\frac{f_{N}(x)}{f_{P}(x)} & = & \frac{f(x)-\pi_{P}f_{P}(x)}{\left(1-\pi_{P}\right)f_{P}(x)}\\
 & = & \frac{f(x)-f(x)\Phi^{*}(x)}{\left(1-\pi_{P}\right)f_{P}(x)}\\
 & = & 0,
\end{eqnarray*}
for all $x\in\mathcal{A}$, which implies that (\ref{eq:anchor})
also holds.Proof of Theorem \ref{thm:variational}\label{subsec:Proof-of-Theorem-variational}
\end{proof}
According to (\ref{eq:f_phi}) and the definition of KL divergence,
\begin{eqnarray*}
\mathrm{KL}(f_{P}||f_{\Phi}) & = & \mathbb{E}_{f_{P}}\left[\log\frac{f_{P}(x)}{f_{\Phi}(x)}\right]\\
 & = & \mathbb{E}_{f_{P}}\left[\log\Phi^{*}(x)\right]+\mathbb{E}_{f_{P}}\left[\log f(x)\right]-\log\mathbb{E}_{f}\left[\Phi^{*}(x)\right]\\
 &  & -\left(\mathbb{E}_{f_{P}}\left[\log\Phi(x)\right]+\mathbb{E}_{f_{P}}\left[\log f(x)\right]-\log\mathbb{E}_{f}\left[\Phi(x)\right]\right)\\
 & = & -\mathcal{L}_{\mathrm{var}}\left(\Phi^{*}\right)+\mathcal{L}_{\mathrm{var}}\left(\Phi\right).
\end{eqnarray*}

\subsection{Analysis of minimum points of $\mathcal{L}_{\mathrm{var}}$\label{subsec:proof-Scale-invariance}}
\begin{prop}
For any constant $c>0$, $\mathcal{L}_{\mathrm{var}}\left(c\cdot\Phi\right)=\mathcal{L}_{\mathrm{var}}\left(\Phi\right)$.
\end{prop}

\begin{proof}
From the definition of $\mathcal{L}_{\mathrm{var}}$, we get
\begin{eqnarray*}
\mathcal{L}_{\mathrm{var}}\left(c\cdot\Phi\right) & = & \log\mathbb{E}_{f}\left[c\cdot\Phi(x)\right]-\mathbb{E}_{f_{P}}\left[\log\left(c\cdot\Phi(x)\right)\right]\\
 & = & \log\int cf(x)\Phi(x)\mathrm{d}x-\int f_{P}(x)\log\left(c\cdot\Phi(x)\right)\mathrm{d}x\\
 & = & \log c+\log\int f(x)\Phi(x)\mathrm{d}x\\
 &  & -\left(\log c\right)\cdot\int f_{P}(x)\mathrm{d}x-\int f_{P}(x)\log\Phi(x)\mathrm{d}x\\
 & = & \log\int f(x)\Phi(x)\mathrm{d}x-\int f_{P}(x)\log\Phi(x)\mathrm{d}x\\
 & = & \mathcal{L}_{\mathrm{var}}\left(\Phi\right)
\end{eqnarray*}
for any $c>0$.
\end{proof}
It can be seen from the above proposition that $\Phi^{*}$ is not
the unique minimum point of $\mathcal{L}_{\mathrm{var}}$. For all
$c\in(0,\frac{1}{\sup_{x}\Phi^{*}(x)}]$, $\Phi=c\cdot\Phi^{*}$ satisfies
\[
0<\Phi(x)\le\frac{\Phi^{*}(x)}{\sup_{x}\Phi^{*}(x)}\le1
\]
and is also a minimum point. In fact, we can show that all minimum
points of $\mathcal{L}_{\mathrm{var}}$ are in the form of $\Phi=c\cdot\Phi^{*}$.
\begin{prop}
\label{prop:optimal-Lvar}A function $\Phi:\mathbb{R}^{d}\mapsto[0,1]$
satisfies $\mathcal{L}_{\mathrm{var}}(\Phi)=\mathcal{L}_{\mathrm{var}}(\Phi^{*})$
iff $\Phi=c\cdot\Phi^{*}$and $c\in(0,\frac{1}{\sup_{x}\Phi^{*}(x)}]$.
\end{prop}

\begin{proof}
The sufficiency is trivial, and so we only give the proof of the necessity.
Suppose that $\Phi$ is a minimum point of $\mathcal{L}_{\mathrm{var}}$,
then
\begin{eqnarray*}
\left.\frac{\partial}{\partial\epsilon}\mathcal{L}_{\mathrm{var}}(\Phi+\epsilon h)\right|_{\epsilon=0} & = & \frac{\partial}{\partial\epsilon}\log\int f(x)\left(\Phi(x)+\epsilon h(x)\right)\mathrm{d}x\\
 &  & -\frac{\partial}{\partial\epsilon}\int f_{P}(x)\log\left(\Phi(x)+\epsilon h(x)\right)\mathrm{d}x\\
 & = & \int\frac{f(x)h(x)}{\mathbb{E}_{f}[\Phi]}-\frac{f_{P}(x)h(x)}{\Phi(x)}\mathrm{d}x
\end{eqnarray*}
must be zero for an arbitrary function $h(x)$. Hence,
\[
\frac{f(x)}{\mathbb{E}_{f}[\Phi]}-\frac{f_{P}(x)}{\Phi(x)}\equiv0\qquad\qquad
\]
\[
\Rightarrow\Phi(x)\equiv\mathbb{E}_{f}[\Phi]\frac{f_{P}(x)}{f(x)}
\]
By combining the above equation and the Bayes rule, we have
\begin{eqnarray*}
\Phi(x) & = & \mathbb{E}_{f}[\Phi]\frac{f_{P}(x)}{f(x)}\\
 & = & \frac{\mathbb{E}_{f}[\Phi]}{\pi_{P}}\cdot\Phi^{*}(x).
\end{eqnarray*}
It is obvious that $c=\frac{\mathbb{E}_{f}[\Phi]}{\pi_{P}}>0$. In
addition, we can conclude from $\sup_{x}\Phi(x)\le1$ that
\[
\sup_{x}c\cdot\Phi^{*}(x)\le1\Rightarrow c\le\frac{1}{\sup_{x}\Phi^{*}(x)}.
\]
\end{proof}
Based on the above analysis, we can conclude that $\Phi^{*}$ can
be uniquely determined for given $f$ and $f_{P}$ under Assumption
\ref{assu:max-Phi}.
\begin{prop}
\label{prop:uniqueness}If Assumption \ref{assu:max-Phi} is satisfied
and $\sup_{x}\Phi(x)=1$ for a function $\Phi:\mathbb{R}^{d}\mapsto[0,1]$,
$\mathcal{L}_{\mathrm{var}}(\Phi)=\mathcal{L}_{\mathrm{var}}(\Phi^{*})$
iff $\Phi=\Phi^{*}$.
\end{prop}

\begin{proof}
This is a trivial corollary of Proposition \ref{prop:optimal-Lvar}.
\end{proof}
Furthermore, the following proposition provides the optimal solutions
in the case where only estimated $f,f_{P}$ are available.
\begin{prop}
\label{prop:bias-distribution}All solutions to $\min_{\Phi}\hat{\mathcal{L}}_{\mathrm{var}}(\Phi)$
with $\hat{\mathcal{L}}_{\mathrm{var}}(\Phi)=\log\mathbb{E}_{\hat{f}}\left[\Phi(x)\right]-\mathbb{E}_{\hat{f}_{P}}\left[\log\left(\Phi(x)\right)\right]$
satisfy
\[
\Phi(x)\propto\hat{f}_{P}(x)/\hat{f}(x).
\]
\end{prop}

\begin{proof}
Omitted as it is similar to that of Proposition \ref{prop:optimal-Lvar}.
\end{proof}
Analysis of Regularization\label{subsec:proof-regularization}

For given $\mathcal{P}$ and $\mathcal{U}$, the empirical estimate
of $\mathcal{L}_{\mathrm{var}}(\Phi)$ is
\[
\hat{\mathcal{L}}_{\mathrm{var}}(\Phi)=\log\frac{1}{N}\sum_{x\in\mathcal{U}}\Phi(x)-\frac{1}{M}\sum_{x\in\mathcal{P}}\log\Phi(x).
\]
Therefore, if the capacity of the model $\Phi$ is extremely high,
simply minimizing $\hat{\mathcal{L}}_{\mathrm{var}}(\Phi)$ yields
\begin{equation}
\Phi(x)=\left\{ \begin{array}{ll}
1, & x\in\mathcal{P},\\
0, & \text{otherwise}.
\end{array}\right.\label{eq:overfitting}
\end{equation}

This overfitting issue can be partly alleviated by early stopping,
i.e., stopping the training when $\mathcal{L}_{\mathrm{var}}(\Phi)$
estimated on the validation set starts to increase. But according
to our numerical experience, it can be more effectively overcome by
the \emph{MixUp} based regularization described in Section \ref{subsec:Regularized-learning-method}.

For two randomly selected $x'\in\mathcal{P}$ and $x''\in\mathcal{U}$,
if $\Phi^{*}(\tilde{x})$ is extremely underestimated with $\Phi(\tilde{x})\to0$
for the virtual sample $\tilde{x}$ (see (\ref{eq:mixup})), we can
conclude that the regularization w.r.t.~$\tilde{x}$
\begin{eqnarray*}
\left(\log\tilde{\Phi}-\log\Phi(\tilde{x})\right)^{2} & \ge & \left(\log\gamma-\log\Phi(\tilde{x})\right)^{2}\\
 & = & \mathcal{O}\left(\left(\log\Phi(\tilde{x})\right)^{2}\right)\to\infty
\end{eqnarray*}
as $\Phi(\tilde{x})\to0$. Thus, with the regularization (\ref{eq:reg}),
the resulting $\Phi(x)$ decay smoothly outside of $\mathcal{P}$
and the trivial solution (\ref{eq:overfitting}) is excluded.

Another possible choice is the mean square error based regularization
$\mathbb{E}_{\tilde{\Phi},\tilde{x}}\left[\left(\tilde{\Phi}-\Phi(\tilde{x})\right)^{2}\right]$,
but this regularization term is bounded and penalizes less for overfitting.

We can also define the regularization by using the standard cross-entropy
loss, which yields the regularization loss
\[
-\tilde{\Phi}\log\Phi(\tilde{x})-\left(1-\tilde{\Phi}\right)\log\left(1-\Phi(\tilde{x})\right)=\mathcal{O}\left(-\log\Phi(\tilde{x})\right)
\]
for each $\tilde{x}$. It can be seen that the proposed mean squared
logarithmic error based regularization penalizes more heavily the
underestimation of $\Phi(\tilde{x})$.

Another possible choice is the mean square error based regularization
$\mathbb{E}_{\tilde{\Phi},\tilde{x}}\left[\left(\tilde{\Phi}-\Phi(\tilde{x})\right)^{2}\right]$,
but this regularization term is bounded and penalizes less for overfitting
has less penalization.

As for the MixUp strategy, the MixUp between $\mathcal{P}$ and $\mathcal{U}$
ensures that $\tilde{\Phi}\approx1$ for $\gamma\approx1$ and $\tilde{\Phi}>\Phi(x'')$
(see (\ref{eq:mixup})), so it can solve the overfitting problem by
penalizing the underestimation of $\Phi(x)$ heavily for unlabeled
data. As a comparison, MixUp inside $\mathcal{P}$ or $\mathcal{U}$
cannot effectively penalize the underestimation of $\Phi(x)$ outside
of $\mathcal{P}$. So we implement MixUp between $\mathcal{P}$ and
$\mathcal{U}$ as in (\ref{eq:mixup}), and can lead to more accurate
and robust classifier according to our numerical experience than MixUp
on $\mathcal{P}\cup\mathcal{U}$ (i.e., $x'$ and $x''$ are both
randomly drawn from $\mathcal{P}\cup\mathcal{U}$) according to our
numerical experience.

The advantage of (\ref{eq:reg}) is demonstrated in Section \ref{subsec:ablation-study}.

\subsection{Proof of Theorem \ref{thm:correctness}\label{subsec:Proof-of-correctness}}

Notice that the variational loss estimated from data
\begin{eqnarray*}
\hat{\mathcal{L}}_{\mathrm{var}}(\Phi(\cdot,\theta)) & = & \log\frac{\sum_{x\in\mathcal{U}}\Phi(x,\theta)}{N}-\frac{\sum_{x\in\mathcal{P}}\log\Phi(x,\theta)}{M}\\
 & \stackrel{p}{\to} & \mathcal{L}_{\mathrm{var}}\left(\Phi(\cdot,\theta)\right)
\end{eqnarray*}
for a given $\theta$ as $M,N\to\infty$. According to Theorem 2.1
in \cite{newey1994large} and Proposition \ref{prop:uniqueness},
we can conclude that the optimal solution $\Phi(x,\theta)$ to (\ref{eq:L})
converges to $\Phi(x,\theta^{*})$ when $M,N\to\infty$ and $\lambda\to0$.

\subsection{Proof of Theorem \ref{thm:bias}\label{subsec:proof-bias-selection}}

By considering condition (ii) in Assumption \ref{assu:separability}
and the fact that $\Phi^{*}(x)$ can be written as
\[
\Phi^{*}(x)=Z^{-1}f_{P}(x)/f(x),
\]
we have

\begin{eqnarray*}
\max_{x}\Phi^{*}(x) & = & Z^{-1}\max_{x}f_{P}(x)/f(x)\in[1-\epsilon,1]\\
\Rightarrow Z & \in & \left[\max_{x}f_{P}(x)/f(x),\frac{\max_{x}f_{P}(x)/f(x)}{1-\epsilon}\right].
\end{eqnarray*}
It can then be known from Proposition \ref{prop:bias-distribution}
that the optimal solution $\Phi$ to
\[
\min\mathcal{L}_{\mathrm{var}}^{\prime}(\Phi)=\log\mathbb{E}_{f}[\Phi(x)]-\mathbb{E}_{f_{P}^{\prime}}[\log\Phi(x)]
\]
under constraint $\max_{x}\Phi(x)=1$ is given by
\[
\Phi(x)=\frac{f_{P}^{\prime}(x)/f(x)}{\max_{x}f_{P}^{\prime}(x)/f(x)}.
\]
We can obtain from condition (i) in Assumption \ref{assu:separability}
that
\begin{eqnarray*}
\Phi(x) & \ge & \frac{c_{1}f_{P}(x)/f(x)}{c_{2}\max f_{P}(x)/f(x)}\\
 & \ge & \frac{c_{1}f_{P}(x)/f(x)}{c_{2}Z}\\
 & = & \frac{c_{1}}{c_{2}}\Phi^{*}(x)
\end{eqnarray*}
and
\begin{eqnarray*}
\Phi(x) & \le & \frac{c_{2}f_{P}(x)/f(x)}{c_{1}\max f_{P}(x)/f(x)}\\
 & \le & \frac{c_{2}f_{P}(x)/f(x)}{c_{1}(1-\epsilon)Z}\\
 & = & \frac{c_{2}}{c_{1}(1-\epsilon)}\Phi^{*}(x)
\end{eqnarray*}

For convenience of analysis, we denote the misclassification probability
of $\Phi$ for a given sample $x$ by

\[
\mathcal{R}_{x}(\Phi)=\left\{ \begin{array}{ll}
\mathbb{P}\left(y=+1|x\right), & \text{if }\Phi(x)<0.5\\
\mathbb{P}\left(y=-1|x\right), & \text{if }\Phi(x)\ge0.5
\end{array}\right..
\]
Thus,

\begin{eqnarray*}
\mathcal{R}_{x}(\Phi)-\mathcal{R}_{x}(\Phi^{*}) & = & \Phi^{*}(x)\cdot1_{\Phi(x)<0.5}+\left(1-\Phi^{*}(x)\right)1_{\Phi(x)\ge0.5}\\
 &  & -\Phi^{*}(x)\cdot1_{\Phi^{*}(x)<0.5}-\left(1-\Phi^{*}(x)\right)1_{\Phi^{*}(x)\ge0.5}\\
 & = & \left(2\Phi^{*}(x)-1\right)\cdot1_{\Phi(x)<0.5}\cdot1_{\Phi^{*}(x)\ge0.5}\\
 &  & +\left(1-2\Phi^{*}(x)\right)\cdot1_{\Phi(x)\ge0.5}\cdot1_{\Phi^{*}(x)<0.5}\\
 & \le & \left(\frac{\Phi^{*}(x)}{\Phi(x)}-1\right)\cdot1_{\Phi(x)<0.5}\cdot1_{\Phi^{*}(x)\ge0.5}\\
 &  & +\left(1-\frac{\Phi^{*}(x)}{\Phi(x)}\right)\cdot1_{\Phi(x)<0.5}\cdot1_{\Phi^{*}(x)\ge0.5}\\
 & \le & \left(\frac{c_{2}}{c_{1}}-1\right)\cdot1_{\Phi(x)<0.5}\cdot1_{\Phi^{*}(x)\ge0.5}\\
 &  & +\left(1-\frac{c_{1}(1-\epsilon)}{c_{2}}\right)\cdot1_{\Phi(x)<0.5}\cdot1_{\Phi^{*}(x)\ge0.5}
\end{eqnarray*}
and
\begin{eqnarray*}
\mathcal{R}(\Phi)-\mathcal{R}(\Phi^{*}) & = & \mathbb{E}\left[\mathcal{R}_{x}(\Phi)-\mathcal{R}_{x}(\Phi^{*})\right]\\
 & \le & \max\left\{ \frac{c_{2}}{c_{1}}-1,1-\frac{c_{1}(1-\epsilon)}{c_{2}}\right\} .
\end{eqnarray*}

\section{Experiment details\label{sec:Experiments-details}}

The data sets are divided into training and test sets. For VPU, a
cross-validation criterion is provided, so we further proportionally
divide the training set into training and validation sets.

For each experiment, 10 repeated runs are done, and mean and standard
variance of test accuracy are calculated. By default, for each run
the neural network is trained for 50 epochs, and results are reported
at the epoch with lowest \textit{\emph{Kullback-Leibler}} divergence
on the validation set. We fix $\alpha$ to $0.3$ and use the \textit{\emph{Kullback-Leibler}}
divergence on the validation set as the criterion for tuning $\lambda$,
selected in \{$1e-4,3e-4,1e-3,\cdots,1,3$\}.

Moreover, we denote $\mathbb{P}(y=+1)$, $\mathbb{P}(y=-1)$ by $\pi_{P}$
and $\pi_{N}$.

\subsection{UCI datasets}

We first clarify the UCI datasets used in our experiments in Table
\ref{tab:Description-of-UCI}. Then, we give the detailed experimental
settings of each experiment in Table \ref{tab:Experimental-settings-UCI}.
The datasets do not go through any preprocessing.

\begin{table}
\caption{\label{tab:Description-of-UCI}Description of UCI datasets used in
experiments.}

\begin{center}
\begin{tabular}{ccccc} 
\toprule
Dataset &  $N$ & size of test set & $d$  \\ 
\midrule 
\footnotesize
Page Blocks &  \scriptsize$3284$ & \scriptsize$2189$ & \scriptsize$10$ \\
Grid Stability &  \scriptsize$6000$ & \scriptsize$4000$ & \scriptsize$14$\\
Avila &  \scriptsize$10430$ & \scriptsize$10437$ & \scriptsize$10$\\
\bottomrule 
\end{tabular}
\end{center}
\end{table}

\begin{table}
\begin{centering}
\caption{\label{tab:Experimental-settings-UCI}Experimental settings for UCI
datasets. $N_{P},M,M_{v}$ denote respectively the number of positive
samples in the training set, number of labeled positive samples in
the training set, number of labeled positive samples in the validation
set. The size of validation unlabeled samples can be calculated via
$N_{v}=N\times M_{v}/M$, where $N$ is the size of training unlabeled
samples.}
\par\end{centering}
\begin{center}
\begin{tabular}{ccccccc}
\toprule
Experiment & setting & Data amount & Validation size & $\pi_P$ & Hyperparameter \\
\midrule
\footnotesize
Page Blocks$^1$ & \scriptsize'2,3,4,5' vs '1' & \scriptsize$N_P$=$342$ $M$=$100$ &  \scriptsize$M_v$=$16$ & \scriptsize$0.104$ & \scriptsize$\lambda=0.0003,\alpha=0.3$\\
\footnotesize
Page Blocks$^2$ & \scriptsize'1' vs '2,3,4,5' & \scriptsize$N_P$=$2942$ $M$=$100$ &  \scriptsize$M_v$=$16$ & \scriptsize$0.896$ & \scriptsize$\lambda=0.0001,\alpha=0.3$\\
\midrule
\footnotesize
Grid Stability$^1$ & \scriptsize'stable' vs 'unstable' & \scriptsize$N_P$=$2187$ $M$=$1000$ &  \scriptsize$M_v$=$167$ & \scriptsize$0.365$ & \scriptsize$\lambda=0.1,\alpha=0.3$\\
\footnotesize
Grid Stability$^2$ & \scriptsize'unstable' vs 'stable' & \scriptsize$N_P$=$3813$  $M$=$1000$ &  \scriptsize$M_v$=$167$ & \scriptsize$0.635$ & \scriptsize$\lambda=0.1,\alpha=0.3$\\
\midrule
\footnotesize
Avila$^1$ & \scriptsize'A' vs The rest & \scriptsize$N_P$=$4286$ $M$=$2000$ &  \scriptsize$M_v$=$192$ & \scriptsize$0.411$ & \scriptsize$\lambda=0.1,\alpha=0.3$\\
\footnotesize
Avila$^2$ & \scriptsize'A,F' vs The rest & \scriptsize$N_P$=$6247$ $M$=$2000$ &  \scriptsize$M_v$=$192$ & \scriptsize$0.599$ & \scriptsize$\lambda=0.03,\alpha=0.3$\\
\bottomrule
\end{tabular}
\end{center}
\end{table}

\subsection{FashionMNIST, CIFAR-10 and STL-10}

Labels of ten classes of each image datasets are reported in Table.~\ref{tab:image-dataset-numbers},
which are denoted by numbers $0$ to $9$ in Section \ref{subsec:Image-datasets}.
The details of the experiments are shown in Table \ref{tab:Experimental-settings-for image data}.
All datasets conduct data preprocessing: normalization with mean and
standard deviation both as $0.5$ at all dimensions.

\begin{table}
\begin{center}
\begin{tabular}{cc}
\toprule
FashionMNIST & t-shirt, trouser, pullover, dress, coat, sandal, shirt, sneaker, bag, ankle boot \\
\midrule
CIFAR-10 & airplane, automobile, bird, cat, deer, dog, frog, horse, ship, truck\\
\midrule
STL-10 & airplane, bird, car, cat, deer, dog, horse, monkey, ship, truck\\
\bottomrule
\end{tabular}
\end{center}

\caption{\label{tab:image-dataset-numbers}Class labels of image datasets,
which are denoted by numbers $0,1,\ldots9$ in Section \ref{subsec:Image-datasets}.}
\end{table}

\begin{table}
\caption{\label{tab:Experimental-settings-for image data}Experimental settings
for FashionMNIST, CIFAR-10 and STL-10. $N_{P},M,M_{v}$ denote respectively
the number of positive samples in the training set, number of labeled
positive samples in the training set, number of labeled positive samples
in the validation set. The size of validation unlabeled samples can
be calculated via $N_{v}=N\times M_{v}/M$, where $N$ is the size
of training unlabeled samples.}

\begin{center}
\begin{tabular}{ccccccc}
\toprule
Experiment & Setting & Data amount & Validation size& $\pi_P$ &Hyperparameter\\
\midrule

FashionMNIST$^1$ & \scriptsize'1,4,7' vs '0,2,3,5,6,8,9' & \scriptsize$N_P$=$15000$ $M$=$3000$ &  \scriptsize$M_v$=$500$ & \scriptsize$0.300$ &  \scriptsize$\lambda=0.3,\alpha=0.3$\\
FashionMNIST$^2$ & \scriptsize'0,2,3,5,6,8,9' vs '1,4,7' & \scriptsize$N_P$=$39000$ $M$=$3000$ &  \scriptsize$M_v$=$500$  & \scriptsize$0.700$ &  \scriptsize$\lambda=3,\alpha=0.3$\\
\midrule

CIFAR-10$^1$ & \scriptsize'0,1,8,9' vs '2,3,4,5,6,7' & \scriptsize$N_P$=$17000$ $M$=$3000$ &  \scriptsize$M_v$=$500$ & \scriptsize$0.400$ &  \scriptsize$\lambda=0.03,\alpha=0.3$\\
CIFAR-10$^2$ & \scriptsize'2,3,4,5,6,7' vs '0,1,8,9' & \scriptsize$N_P$=$27000$ $M$=$3000$ &   \scriptsize$M_v$=$500$ &\scriptsize$0.600$ &  \scriptsize$\lambda=0.01,\alpha=0.3$\\
\midrule

STL-10$^1$ & \scriptsize'0,2,3,8,9' vs '1,4,5,6,7' & \scriptsize$N_P$=$100000$ $M$=$2500$ &  \scriptsize$M_v$=$250$ & \scriptsize$unknown$  &  \scriptsize$\lambda=0.3,\alpha=0.3$\\
STL-10$^2$ & \scriptsize'1,4,5,6,7' vs '0,2,3,8,9' & \scriptsize$N_P$=$100000$ $M$=$2500$ &   \scriptsize$M_v$=$250$ &\scriptsize$unknown$  &  \scriptsize$\lambda=0.1,\alpha=0.3$\\
\bottomrule
\end{tabular}
\end{center}
\end{table}

\subsection{Choice of hyperparameters of GenPU\label{subsec:Choice-of-hyperparameters-GenPU}}

GenPU contains four hyperparameters: $\pi_{P}\lambda_{p}$, $\pi_{P}\lambda_{u}$,
$\pi_{N}\lambda_{n}$, $\pi_{N}\lambda_{u}$. Although the parameters
are coupled for given $\pi_{P}$ in \cite{hou2017generative}, our
experience shows that the better performance can be achieved by selecting
the four parameters independently. Table \ref{tab:GenPU-parameters}
shows the best hyperparameters which lead to the largest classification
accuracies on test sets. They are selected in $\{0.01,0.05,0.1,0.5,\ldots,1000,5000\}$
by greedy grid search.

\begin{table}
\caption{\label{tab:GenPU-parameters}Choice of hyperparameters for GenPU.}

\begin{center}
\begin{tabular}{cccccc} 
\toprule
Dataset & $\pi_{P}\lambda_{p}$ & $\pi_{P}\lambda_{u}$ & $\pi_{N}\lambda_{n}$& $\pi_{N}\lambda_{u}$\\
\midrule 
FashionMNIST & $0.01$ & $1$ & $100$ & $1$\\
& $0.01$ & $1$ & $1000$ & $50$\\
\midrule
CIFAR-10 & $0.01$ & $1$ & $100$ & $1$\\
& $0.01$ & $1$ & $100$ & $1$\\
\midrule
Page Blocks & $0.01$ & $1$ & $1000$ & $1$\\
& $0.01$ & $1$ & $200$ & $1$\\
\midrule
Grid Stability & $0.01$ & $1$ & $1000$ & $500$\\
& $0.01$ & $1$ & $1000$ & $500$\\
\midrule
Avila & $0.01$ & $1$ & $100$ & $1$\\
& $0.001$ & $1$ & $1000$ & $500$\\
\bottomrule 
\end{tabular}
\end{center}
\end{table}

\subsection{Comparison with known $\pi_{P}$}

In Table \ref{tab:UCI-Image-exact-pi}, we compare the classification
accuracies of VPU, nnPU and uPU on UCI and image datasets. All the
settings are the same as in the main body of the paper, except that
the true value of $\pi_{P}$ is assumed to be known for nnPU and uPU.
Notice that the experiment on STL-10 is not performed because the
exact $\pi_{P}$ is unavailable.

\begin{table}[tbh]
\caption{\label{tab:UCI-Image-exact-pi}Classification accuracies (\%) of compared
methods, where $*$ means that the algorithm is performed with the
true value of $\pi_{P}$.}

\footnotesize
\begin{center}
\begin{tabular}{cccccccccc}  
\toprule  
Dataset  
& Page Blocks$^1$ & Page Blocks$^2$ & Grid Stability$^1$ &  Grid Stability$^2$ & Avila$^1$& Avila$^2$\\
\midrule 
\footnotesize
VPU &
\scriptsize \bm{$93.6\pm0.4$} & \scriptsize \bm{$93.5\pm0.7}$ & \scriptsize \bm{$92.6\pm0.3$} &  \scriptsize $89.5\pm0.5$ & \scriptsize \bm{$82.0\pm0.9$} & \scriptsize \bm{$87.2\pm0.5$}\\
\midrule
\footnotesize
nnPU$^*$ &
\scriptsize$92.3\pm1.2$ & \scriptsize$91.7\pm0.6$ & \scriptsize$91.5\pm1.7$ & \scriptsize \bm{$90.5\pm0.3$} & \scriptsize $75.9\pm2.2$ &\scriptsize $84.8\pm0.5$ \\
\midrule
\footnotesize
uPU$^*$ &
\scriptsize $93.0\pm1.2$ &\scriptsize $90.0\pm2.8$ &\scriptsize $92.2\pm0.1$ & \scriptsize$87.9\pm0.9$ & \scriptsize$76.5\pm1.0$ & \scriptsize$84.0\pm1.0$ \\
\bottomrule 

\toprule  
Dataset  & F-MNIST$^1$ & F-MNIST$^2$ & CIFAR-10$^1$ & CIFAR-10$^2$ &  &   \\
\midrule 
\footnotesize
VPU & \scriptsize \bm{$92.7\pm0.3$} & \scriptsize \bm{$90.8\pm0.6$} & \scriptsize\bm{$89.5\pm0.1$} & \scriptsize\bm{$88.8\pm0.8$} &  &  \\
\midrule
\footnotesize
nnPU$^*$ & \scriptsize$92.1\pm0.3$ & \scriptsize$90.7\pm1.4$ & \scriptsize$87.2\pm0.7$  & \scriptsize$86.5\pm1.7$ &  &    \\
\midrule
\footnotesize
uPU$^*$ &
\scriptsize$90.4\pm1.4$ & \scriptsize$74.1\pm1.9$ & \scriptsize$79.1\pm2.4$ & \scriptsize$68.7\pm0.4$ &  &  \\
\bottomrule 
\end{tabular}
\end{center}
\end{table}

\subsection{Mode collapse of GenPU \label{subsec:GenPU-mode-collapse}}

The failure of GenPU in the experiments is caused by mode collapse.
This is demonstrated in Fig. \ref{fig:mode_collapse}, which shows
the positive (a) and negative (b) images generated by GenPU. Positive
labels (\textquoteright Positive\textquoteright{} vs \textquoteright Negative\textquoteright )
are given by \textquoteright 1,4,7\textquoteright{} (Trouser, Coat,
Sneaker) vs \textquoteright 0,2,3,5,6,8,9\textquoteright{} (T-shirt/Top,
Pullover, Dress, Sandal, Skirt, Bag, Ankle boot). We observe that,
in spite of the good quality of the generated images, some modes are
neglected be the generators.

\begin{figure}
\centering{}\includegraphics[width=0.5\textheight]{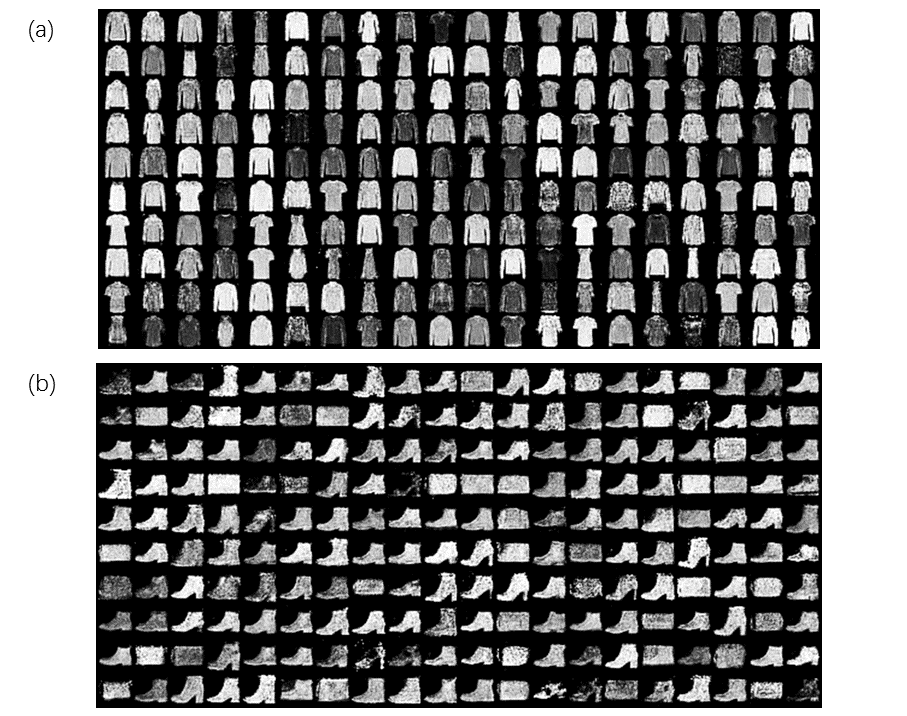}\caption{\label{fig:mode_collapse}Positive (a) and negative (b) samples generated
by GenPU on FashionMNIST with '1, 4, 7' as positive labels}
\end{figure}

\subsection{KM2, nnPU and uPU under selection bias \label{subsec:selection-bias}}

Table \ref{tab:bias_KM2} shows that the class prior estimation method
KM2 significantly affected by the selection bias, which also yields
poor performance of nnPU. As can be observed in Fig. \ref{fig:bias-acc-pi},
nnPU is even more robust to selection bias if the accurate $\pi_{P}$
is known a priori.

\begin{table}
\caption{\label{tab:bias_KM2}The class prior estimated by KM2 under selection
bias with the true class prior $\text{\ensuremath{\pi_{P}}}=0.3$}

\begin{center}
\begin{tabular}{ccccccccccc} 
\toprule 
$n_1/n_4$  & 1 & 2 & 3 & 4 & 5 & 6 & 7 & 8 & 9 & 10 \\ 
\midrule
estimated ${\pi}_P$ & 0.267  & 0.249 & 0.206 & 0.188 & 0.164 & 0.170 & 0.151 & 0.157 & 0.150 & 0.144 \\ 
\bottomrule 
\end{tabular} 
\end{center}
\end{table}

\begin{figure}
\begin{centering}
\includegraphics[scale=0.5]{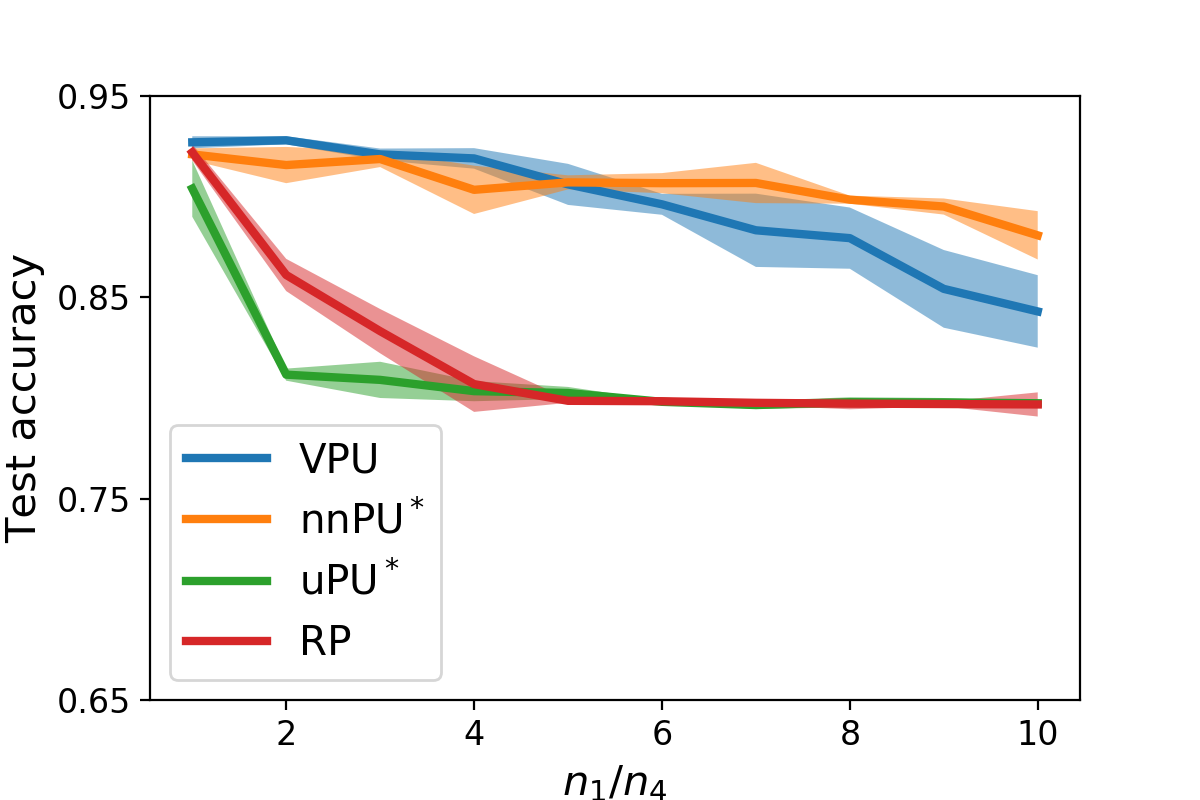}
\par\end{centering}
\caption{\label{fig:bias-acc-pi}Comparison of PU methods under selection bias
of $\mathcal{P}$, with accurate class prior $\pi_{P}$ known for
uPU and nnPU}
\end{figure}

\subsection{Alternative regularization terms\label{subsec:Alternative-regularization-terms}}

Mixup is a powerful regularization technique, but it might not be
applicable to domains other than image. Besides, its data-augmentation
nature undermines credibility of VPU's superiority shown in the experiments.
In fact, some other forms of regularization also work well, such as
adversarial training \cite{goodfellow2014explaining} and virtual
adversarial training \cite{miyato2018virtual}. Here we introduce
a large-margin regularization term, proposed in \cite{Wang2018Additive},
as an alternative for the Mixup-based regularization. It penalizes
the positive instances that are misclassified by \textgreek{F} or
have small margins between $\log\Phi\left(x\right)$ and $\log\left(1-\Phi\left(x\right)\right)$.
It is a smooth version of $\max\left\{ 0,\log\left(1-\Phi\left(x\right)\right)+\log\alpha-\log\Phi\left(x\right)\right\} $
and formulates as 
\begin{align}
\mathcal{L}_{reg-margin}\left(\Phi\right) & =\text{softplus}\left(\log\left(1-\Phi\left(x\right)\right)+\log\alpha-\log\Phi\left(x\right)\right)\nonumber \\
 & =\log\left(1+\alpha\frac{1-\Phi\left(x\right)}{\Phi\left(x\right)}\right).\label{eq:L_margin}
\end{align}
Table \ref{tab:acc-margin} reports the results of experiments with
the same setting as in Section \ref{subsec:UCI-experiments} and \ref{subsec:Image-datasets}.
Though not as good as the Mixup-based regularization, the large-margin
regularization significantly outperforms nnPU in most experiments.

\begin{table}
\caption{\label{tab:acc-margin}Classification accuracies (\%) on image and
UCI datasets of experiments with the same setting as in Section \ref{subsec:UCI-experiments}
and \ref{subsec:Image-datasets}. VPU w/ Mixup is the VPU we develop
in main body of this paper, while VPU w/ margin replaces the regularization
with the large-margin loss (\ref{eq:L_margin}).}

\begin{center}
\footnotesize
\begin{tabular}{ccccccc}  
\toprule  
Dataset  
& Page Blocks$^1$ & Page Blocks$^2$ & Grid Stability$^1$&  Grid Stability$^2$ & Avila$^1$& Avila$^2$ \\
\midrule 
\footnotesize
VPU w/ Mixup&
\scriptsize $93.6\pm0.4$ & \scriptsize $93.5\pm0.7$ & \scriptsize \bm{$92.6\pm0.3}$ &  \scriptsize $89.5\pm0.5$ & \scriptsize \bm{$82.0\pm0.9$} & \scriptsize \bm{$87.2\pm0.5$} \\
\midrule 
\footnotesize
VPU w/ margin&
\scriptsize \bm{$95.6\pm1.3$} & \scriptsize \bm{$94.0\pm0.6}$ & \scriptsize \bm{$92.6\pm0.3}$ &  \scriptsize \bm{$90.5\pm0.5$} & \scriptsize $81.4\pm0.3$ & \scriptsize $86.8\pm0.5$ \\
\midrule 
\footnotesize
nnPU &
\scriptsize $93.4\pm1.1$ & \scriptsize$90.2\pm2.6$ &\scriptsize$80.8\pm2.5$ &\scriptsize$84.1\pm1.8$ &\scriptsize$73.3\pm2.0$ &\scriptsize $83.1\pm2.1$ \\
\midrule
Dataset  & F-MNIST$^1$ & F-MNIST$^2$ & CIFAR-10$^1$ & CIFAR-10$^2$ & STL-10$^1$ & STL-10$^2$ \\
\midrule 
\footnotesize
VPU w/ Mixup& \scriptsize \bm{$92.7\pm0.3$} & \scriptsize $90.8\pm0.6$ & \scriptsize\bm{$89.5\pm0.1$} & \scriptsize $88.8\pm0.8$ & \scriptsize\bm{$79.7\pm1.5$} & \scriptsize \bm{$83.7\pm0.1$} \\
\midrule 
\footnotesize
VPU w/ margin& \scriptsize $92.6\pm0.4$ & \scriptsize \bm{$91.1\pm0.2$} & \scriptsize   $89.2\pm0.2$ & \scriptsize\bm{$88.9\pm0.3$} & \scriptsize $74.5\pm0.9$ & \scriptsize $82.6\pm1.5$ \\
\midrule 
\footnotesize
nnPU & \scriptsize$90.8\pm0.6$ & \scriptsize$ 90.5\pm0.4$ & \scriptsize$85.6\pm2.3$ & \scriptsize$85.5\pm2.0$ & \scriptsize $78.3\pm1.2$ & \scriptsize$82.2\pm0.5$ \\
\bottomrule 
\end{tabular}
\end{center}
\end{table}

\subsection{Other metric for comparison\label{subsec:Other-metric-for}}

Accuracy might not be the best metric, especially when data sets are
imbalanced. Therefore, except accuracy shown in the main body, we
here also report in Table \ref{tab:image data-AUC} the area under
curve (AUC) values of experiments on image datasets

\begin{table}[tbh]
\caption{\label{tab:image data-AUC}AUC values of compared methods on FashionMNIST
(abbreviated as ``F-MNIST''), CIFAR-10 and STL-10 datasets. Experiment
settings are the same as in Section \ref{subsec:Image-datasets}.}

\begin{center}  
\begin{tabular}{cccccccc}    
\toprule   
Dataset  & F-MNIST$^1$ & F-MNIST$^2$ & CIFAR-10$^1$ & CIFAR-10$^2$ & STL-10$^1$ & STL-10$^2$ \\   \midrule    
\footnotesize   VPU & \scriptsize \bm{$0.973\pm0.002$} & \scriptsize \bm{$0.957\pm0.005$} & \scriptsize\bm{$0.956\pm0.001$} & \scriptsize\bm{$0.953\pm0.003$} & \scriptsize\bm{$0.976\pm0.002$} & \scriptsize \bm{$0.963\pm0.003$} \\   
\midrule    
\footnotesize   nnPU & \scriptsize$0.961\pm0.004$ & \scriptsize$ 0.945\pm0.005$ & \scriptsize$0.954\pm0.003$ & \scriptsize$0.953\pm0.002$ & \scriptsize $0.850\pm0.007$ & \scriptsize$0.898\pm0.004$ \\   
\midrule    
\footnotesize   uPU & \scriptsize$0.955\pm0.006$ & \scriptsize$0.918\pm0.008$ &\scriptsize$0.952\pm0.003$ & \scriptsize$0.949\pm0.004$ & \scriptsize$0.823\pm0.013$ & \scriptsize$0.862\pm0.014$\\   
\midrule    
\footnotesize   Genpu & \scriptsize$0.673\pm0.018$ & \scriptsize$0.868\pm0.007$ & \scriptsize$0.790\pm0.012$ & \scriptsize$0.811\pm0.014$ & \scriptsize$0.789\pm0.004$ & \scriptsize$0.793\pm0.011$\\   
\midrule    
\footnotesize   RP & \scriptsize \bm{$0.973\pm0.001$} & \scriptsize$0.954\pm0.002$ & \scriptsize$0.953\pm0.002$ & \scriptsize$0.951\pm0.003$ & \scriptsize$0.829\pm0.019$ & \scriptsize$0.851\pm0.015$ \\   
\bottomrule   
\end{tabular} 
\end{center}
\end{table}

\subsection{nnPU with Mixup}

To further demonstrate the advantage of VPU over nnPU, we also conduct
experiments on nnPU on FashionMNIST with unlabeled data augmented
by MixUp. The classification accuraries are reportd in Table \ref{tab:nnpu-with-mixup},
which shows that nnPU does not significantly benefit from Mixup.

\begin{table}[tbh]
\caption{\label{tab:nnpu-with-mixup} Classification accuracies (\%) of nnPU
with Mixup on FashionMNIST. Experiment settings are the same as in
Section \ref{subsec:Image-datasets}. The {*} mark indicates accurate
class prior known.}

\begin{center}  
\begin{tabular}{cccccc}    
\toprule   
&VPU & nnPU & nnPU+MixUp & nnPU$^*$ & nnPU$^*$+MixUp \\
\midrule
F-MNIST$^1$ & \scriptsize $\bm{92.7\pm 0.3}$ & \scriptsize $90.8\pm 0.6$ &\scriptsize $91.0\pm 0.6$ &\scriptsize$92.1\pm 0.3$ &\scriptsize $92.4\pm0.5$\\
\midrule
F-MNIST$^2$ & \scriptsize $\bm{90.8\pm 0.6}$ & \scriptsize $90.5\pm 0.4$ &\scriptsize $89.9\pm 0.3$ &\scriptsize$90.7\pm 1.4$ & \scriptsize$90.7\pm0.4$\\

\bottomrule   
\end{tabular} 
\end{center}
\end{table}

\section{Extension \label{sec:Extension}}

One alternative to the variational loss is
\begin{eqnarray*}
\mathcal{L}_{\mathrm{JS}}(\Phi) & = & \max_{D:\mathbb{R}^{d}\mapsto[0,1]}\mathbb{E}_{f_{P}}\left[\log D(x)\right]+\mathbb{E}_{f_{\Phi}}\left[\log\left(1-D(x)\right)\right],\\
 & = & \max_{D:\mathbb{R}^{d}\mapsto[0,1]}\int f_{P}(x)\log D(x)+f_{\Phi}(x)\log\left(1-D(x)\right)\mathrm{d}x.
\end{eqnarray*}
Here $D$ can be interpreted as a discriminator as in GAN, which intends
to separate the samples drawn from $f_{P}$ and those obtained by
sampling from $f_{\Phi}$. By setting
\[
\frac{\partial\left(f_{P}(x)\log D(x)+f_{\Phi}(x)\log\left(1-D(x)\right)\right)}{\partial D(x)}=0,
\]
we can obtain that the optimal $D$ is
\[
D(x)=\frac{f_{P}(x)}{f_{P}(x)+f_{\Phi}(x)},
\]
and
\begin{eqnarray*}
\mathcal{L}_{\mathrm{JS}}(\Phi) & = & \int f_{P}(x)\log\frac{f_{P}(x)}{\frac{1}{2}\left(f_{P}(x)+f_{\Phi}(x)\right)}\mathrm{d}x+\log\frac{1}{2}\\
 &  & +\int f_{\Phi}(x)\log\frac{f_{\Phi}(x)}{\frac{1}{2}\left(f_{P}(x)+f_{\Phi}(x)\right)}\mathrm{d}x+\log\frac{1}{2}\\
 & = & 2\mathrm{JS}\left(f_{P}||f_{\Phi}\right)-\log4,
\end{eqnarray*}
where $\mathrm{JS}\left(f_{P}||f_{\Phi}\right)$ denotes the Jensen-Shannon
divergence between $f_{P}$ and $f_{\Phi}$. Thus, $\mathcal{L}_{\mathrm{JS}}(\Phi)-\mathcal{L}_{\mathrm{JS}}(\Phi^{*})\ge0$
for all $\Phi$ since $f_{P}=f_{\Phi^{*}}$. In practice, we can approximate
$D$ by another neural network, and minimize $\mathcal{L}_{\mathrm{JS}}$
by adversarial learning.

Another choice of variational loss can be derived from a weighted
$L^{2}$ distance between $f_{\Phi}$ and $f_{P}$ as
\begin{eqnarray*}
\int f(x)^{-1}\left(f_{\Phi}(x)-f_{P}(x)\right)^{2}\mathrm{d}x & = & \frac{\int f(x)\Phi(x)^{2}\mathrm{d}x}{\mathbb{E}_{f}[\Phi(x)]^{2}}-2\frac{\int f_{P}(x)\Phi(x)\mathrm{d}x}{\mathbb{E}_{f}[\Phi(x)]}\\
 &  & +\int f(x)^{-1}f_{P}(x)^{2}\mathrm{d}x,\\
 & = & \frac{\mathbb{E}_{f}[\Phi(x)^{2}]}{\mathbb{E}_{f}[\Phi(x)]^{2}}-2\frac{\mathbb{E}_{f_{P}}[\Phi(x)]}{\mathbb{E}_{f}[\Phi(x)]}\\
 &  & +\int f(x)^{-1}f_{P}(x)^{2}\mathrm{d}x,\\
 & = & \mathcal{L}_{2}(\Phi)+\int f(x)^{-1}f_{P}(x)^{2}\mathrm{d}x,
\end{eqnarray*}
where
\[
\mathcal{L}_{2}(\Phi)\triangleq\frac{\mathbb{E}_{f}[\Phi(x)^{2}]}{\mathbb{E}_{f}[\Phi(x)]^{2}}-2\frac{\mathbb{E}_{f_{P}}[\Phi(x)]}{\mathbb{E}_{f}[\Phi(x)]}
\]
and $\int f(x)^{-1}f_{P}(x)^{2}\mathrm{d}x$ is a constant independent
of $\Phi$. It can be seen from the above that the loss $\mathcal{L}_{2}$
satisfies
\begin{eqnarray*}
\mathcal{L}_{2}(\Phi)-\mathcal{L}_{2}(\Phi^{*}) & = & \int f(x)^{-1}\left(f_{\Phi}(x)-f_{P}(x)\right)^{2}\mathrm{d}x\\
 & \ge & 0.
\end{eqnarray*}

\end{document}